\documentclass[10pt,letterpaper]{article}

\usepackage[margin=1in]{geometry}
\usepackage[utf8]{inputenc}
\usepackage[T1]{fontenc}
\usepackage{lmodern}
\usepackage{newunicodechar}
\newunicodechar{ő}{\H{o}}
\newunicodechar{–}{--}
\newunicodechar{—}{---}
\newunicodechar{’}{'}
\newunicodechar{“}{\textquotedblleft}
\newunicodechar{”}{\textquotedblright}
\usepackage{url}
\usepackage{booktabs}
\usepackage{amsfonts}
\usepackage{nicefrac}
\usepackage{microtype}
\usepackage{xcolor}
\usepackage{amsmath,amssymb}
\usepackage{mathtools}
\usepackage{amsthm}
\usepackage{graphicx}
\usepackage{tabularx}
\usepackage{caption}
\usepackage{subcaption}
\usepackage{multirow}
\usepackage[authoryear,round]{natbib}
\usepackage[hidelinks]{hyperref}

\newcommand\ourmethod{StableRCA}

\theoremstyle{plain}
\newtheorem{theorem}{Theorem}[section]
\newtheorem{proposition}[theorem]{Proposition}

\theoremstyle{definition}

\newtheorem{assumption}[theorem]{Assumption}
\theoremstyle{remark}

\title{\ourmethod: Mechanism-Level Root Cause Analysis without Global Graph Discovery}

\author{%
  Xiaoyu Lin\textsuperscript{1,*}\quad
  Nicholas Tagliapietra\textsuperscript{2,3,*}\quad
  Kehan Li\textsuperscript{1,*}\\
  Lavdim Halilaj\textsuperscript{2}\quad
  Juergen Luettin\textsuperscript{2}\quad
  Peng Cui\textsuperscript{1}\\[0.6em]
  \small\textsuperscript{1}Department of Computer Science, Tsinghua University\\
  \small\textsuperscript{2}Bosch Center for Artificial Intelligence, Renningen, Germany\\
  \small\textsuperscript{3}Computer Science Department, TU Darmstadt, Germany\\[0.5em]
  \footnotesize\href{mailto:xiaoyulin@mail.tsinghua.edu.cn}{\nolinkurl{xiaoyulin@mail.tsinghua.edu.cn}}\quad
  \href{mailto:nicholas.tagliapietra@de.bosch.com}{\nolinkurl{nicholas.tagliapietra@de.bosch.com}}\\
  \footnotesize\href{mailto:lkh20@mails.tsinghua.edu.cn}{\nolinkurl{lkh20@mails.tsinghua.edu.cn}}\quad
  \href{mailto:Lavdim.Halilaj@de.bosch.com}{\nolinkurl{Lavdim.Halilaj@de.bosch.com}}\\
  \footnotesize\href{mailto:Juergen.Luettin@de.bosch.com}{\nolinkurl{Juergen.Luettin@de.bosch.com}}\quad
  \footnotesize\href{mailto:cuip@tsinghua.edu.cn}{\nolinkurl{cuip@tsinghua.edu.cn}}\\
  \small\textsuperscript{*}Equal contribution
}
\date{}

\begin{document}

\maketitle

\begin{abstract}
Root-cause analysis (RCA) seeks to identify the true causal variables that are responsible for a system's abnormal behaviors. A central challenge in this task is to distinguish changes where the anomaly occurs from downstream propagated abnormal effects, particularly when the underlying causal structure is unknown. To address this challenge, we propose StableRCA, an RCA framework that helps distinguish root causes from variables affected by fault propagation by examining how each candidate's conditional distribution, given a selected set of other variables, changes between the normal and anomalous regimes, without requiring knowledge of the underlying causal graph. The method first identifies variables with marginal distribution shifts as root-cause candidates, then selects a local conditioning set for each candidate using stable learning, and finally ranks the candidates by the magnitude of their local conditional shifts, using an importance-weighted Relative Risk Variation (RRV) score as a proxy. Our theoretical analysis establishes sufficient conditions under which comparing the conditional distributions of candidate variables across the normal and anomalous regimes, given suitable conditioning sets, identifies the root causes. Experiments on synthetic benchmarks and four real-world application benchmarks demonstrate \ourmethod's competitive performance across diverse domains.

\end{abstract}

\section{Introduction}\label{sec:introduction}

Root cause analysis (RCA) is a critical task in many high-stakes domains, including manufacturing \citep{10.1007/s10845-022-01914-3, Papageorgiou2022RCAReview}, IT services and cloud systems \citep{solé2017surveymodelstechniquesrootcause, 10.1145/3501297}, healthcare \citep{Kellogg2017RCAPatientSafety}, and medicine \citep{Wu2008RCAJAMA}, where identifying the factors responsible for abnormal system behavior is essential for effective diagnosis and remediation. 
RCA is particularly challenging in modern systems, where large numbers of numerical, categorical, and binary variables exhibit complex interdependencies. When an anomaly originates from a small subset of variables in a system, its effects can propagate through these dependencies and induce distributional changes in many downstream variables. Consequently, a large number of variables may appear abnormal even though only a few correspond to the actual root causes. Distinguishing such causal sources from their propagated effects therefore requires reasoning about the underlying dependency structure. In practice, however, the global underlying structure is typically unknown and difficult to recover reliably from observational data \citep{montagna2023assumption, chevalley2025causalbench}.   

Existing RCA approaches differ in how much structural information they require and how they use it to identify the true causal sources. Specifically, statistical approaches rely on anomaly detection and statistical testing to identify variables associated with abnormal behavior \citep{10.1145/3308558.3313653,pham2024baro}, without explicitly modeling causal dependencies. Such evidence alone, however, may not distinguish changes in a variable's own mechanism from propagated effects. Graph-based methods rely on known, constructed, or learned global graphs to localize root causes through graph traversal, analysis of anomaly propagation, or causal inference \citep{wang2018cloudranger, 10.1145/3366423.3380111, 10.1145/3580305.3599392, 10.1145/3534678.3539041, budhathoki2022causal}. However, an accurate global graph is often unavailable, and estimating it through causal discovery can be challenging in large, complex systems. Methods based on local dependencies or partial causal information reduce reliance on a complete graph by using partial causal graphs or local causal discovery \citep{10.5555/3600270.3602529, ikram2025partialrca}. In this paper, we adopt this local perspective. Our key insight is that under appropriate conditions, local conditional distribution shifts can reveal where an anomaly originates without reconstructing a global causal graph. The challenge is then to select suitable local conditioning sets and extract reliable evidence of conditional distribution shift from limited anomalous data.

Motivated by this challenge, we propose \ourmethod, a framework for mechanism-level RCA that does not require recovery of the global causal structure. We adopt an intervention-based problem formulation in which anomalies arise from interventions on one or more variables upstream of a designated observed outcome, assuming that no intervention target is an ancestor of another. Under suitable distributional conditions, we show that comparing local conditional distributions across normal and abnormal regimes, using appropriately chosen conditioning sets, distinguishes intervention targets from other marginally shifted variables. Building on this result, \ourmethod\ implements three components. First, it filters variables to retain those exhibiting marginal distribution shifts. Second, for each retained variable, stable learning with the Sample Reweighted Decorrelation Operator (SRDO) \citep{shen2020stable, pmlr-v162-xu22o} guides local conditioning-set selection from normal data through independence-driven reweighting, aiming to reduce reliance on spurious predictive associations. Third, an importance-weighted Relative Risk Variation (RRV) score ranks candidates by their relative increase in prediction risk after aligning the conditioning-variable distributions across regimes, providing a practical proxy for local conditional distribution shifts.

Our contributions are summarized as follows:
\begin{itemize}
    \item We introduce \ourmethod, a graph-agnostic framework for mechanism-level RCA that integrates marginal distribution shift identification, stable learning for local conditioning-set selection, and aligned prediction-risk scoring to rank root-cause candidates without reconstructing a global causal graph.

    \item We establish sufficient conditions for root-cause identifiability, showing how marginal and local conditional shifts jointly distinguish intervention targets from variables affected by anomaly propagation.

    \item We evaluate \ourmethod\ on synthetic benchmarks with single and multiple intervention targets and on four real-world application benchmarks across diverse domains. Ablation studies further assess the contributions of its key components.
\end{itemize}

\begin{figure}[tb]
    \centering
    \includegraphics[width=.87\textwidth]{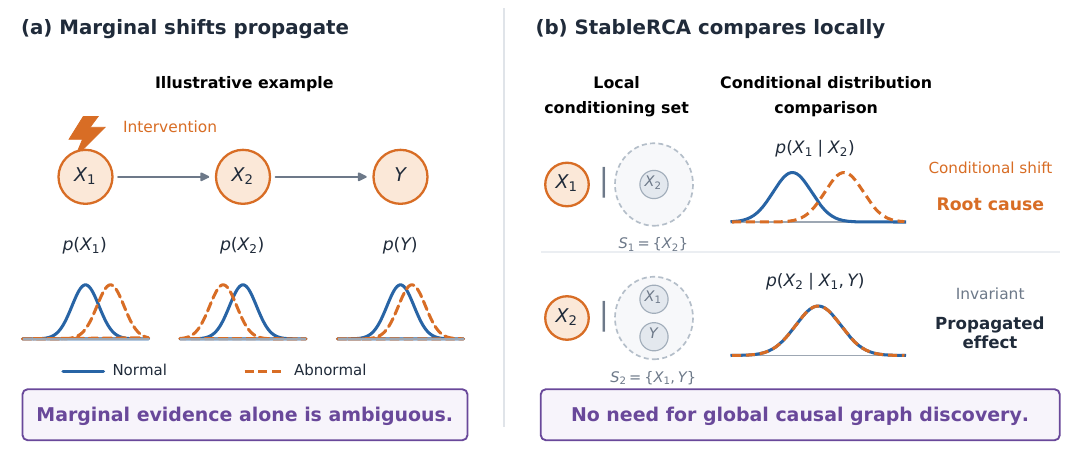}
    \captionsetup{width=\textwidth}
    \caption{\textbf{Motivation for StableRCA.}
    (a) Illustrative example: Marginal distribution shifts can occur in both the root cause $X_1$ and the downstream variable $X_2$ and $Y$, making them difficult to distinguish.
    b) Comparing conditional distributions across regimes using appropriate conditioning variables can distinguish $X_1$ from $X_2$ without reconstructing the full causal graph.}
    \label{fig:motivation_picture}
\end{figure}



\section{Related Work}\label{sec:related_work}

\textbf{Graph-Based RCA.}
Many RCA methods localize root causes using a known, constructed, or learned global graph. Early approaches, especially in microservice systems, build dependency, causal, or impact graphs from domain knowledge, system topology, or observational data, and then rank candidate root causes using graph traversal or scoring procedures such as random walks \citep{wang2018cloudranger, 10.1145/3366423.3380111, 10.1145/3580305.3599392, 10.1145/3589334.3645442}, PageRank \citep{10.1016/j.jss.2023.111724, 10.1609/aaai.v38i1.27772}, or depth-first search \citep{6848128, 10.1007/978-3-030-03596-9_1}. These approaches primarily exploit graph topology rather than explicitly modeling how interventions alter data-generating mechanisms. More recent causal RCA methods instead formulate root-cause localization as intervention-target identification under an Structural Causal
Models (SCM). For example, CIRCA identifies root causes by measuring changes in conditional distributions given parent variables on a causal Bayesian network \citep{10.1145/3534678.3539041}. Under the assumption that the complete SCM is known, \citet{budhathoki2022causal} define the root causes of outliers as interventions on exogenous noise variables, and identify them by simulating counterfactual distributions via noise randomization. Despite the differences in modeling and inference, these graph-based methods ultimately depend on a sufficiently accurate global causal structure, which is often difficult to obtain in high-dimensional real-world settings.

\textbf{RCA Without Full Structural Knowledge.}
Several RCA methods relax the need for a fully known causal graph. RCD introduces an auxiliary intervention indicator and performs hierarchical localized conditional-independence tests to identify root-cause candidates \citep{10.5555/3600270.3602529}. RCG instead uses a partial causal graph learned before a failure to identify possible parents for each variable. It then ranks variables by their conditional mutual information with an indicator distinguishing normal from abnormal data, using each variable's possible parents as the conditioning set \citep{ikram2025partialrca}. Under a single-root-cause polytree graph setting, Score Ordering ranks variables by marginal anomaly scores without a graph; the same work also proposes Smooth Traversal for the graph-known setting, identifying the root cause through parent-child score differences \citep{orchard2025root}. Likewise, \citet{10.1093/jrsssb/qkaf066} establishes identifiability in linear acyclic SCMs with a single mean-shift intervention via an invariance property induced by variable permutations and Cholesky decomposition. 
Overall, some of these methods rely on localized conditional-independence tests or pre-estimated partial graphs, while others provide graph-free identification guarantees under specific graph and intervention assumptions. \ourmethod\ takes a complementary approach, using stable learning to select a conditioning set for each candidate from normal data
and importance-weighted risk comparisons to assess changes in
the candidate's conditional distribution. This supports mechanism-level RCA without requiring a pre-estimated global or partial causal graph.

\textbf{Non-Causal RCA Methods.}
Another line of work approaches RCA through anomaly detection or statistical testing without explicit causal modeling. For example, $\epsilon$-Diagnosis identifies root causes using two-sample testing and $\epsilon$-statistics \citep{10.1145/3308558.3313653}, while BARO combines multivariate Bayesian online change-point detection with nonparametric hypothesis testing for root-cause localization \citep{pham2024baro}. These methods can be computationally efficient and effective for anomaly localization, but without an explicit notion of causal propagation or mechanism invariance, statistical distribution shifts alone generally cannot distinguish intervention targets from variables affected downstream.

\textbf{Invariant Prediction and Stable Learning.}
Invariant prediction connects causal reasoning with robust out-of-distribution generalization by seeking predictive relationships that remain stable across different environments. \citet{Peters2016InvariantPrediction} exploit invariance of conditional distributions to identify causal predictors under appropriate assumptions, while related work explores invariance for transfer learning and out-of-distribution generalization \citep{RojasCarulla2018InvariantModels, Arjovsky2019InvariantRisk}. Stable learning pursues a related goal through sample reweighting that reduces dependence among predictors and limits reliance on spurious associations \citep{Kuang_Xiong_Cui_Athey_Li_2020,shen2019stablelearningsamplereweighting, Zhang_2021_CVPR, pmlr-v162-xu22o}. Drawing on this reweighting principle, \ourmethod\ selects local conditioning sets from normal data and uses them to assess conditional distribution shifts between normal and anomalous regimes.

\section{Problem Formulation}\label{sec:problem_formulation}

\textbf{Notation}
Let $\mathbf{X}=(X_1,\ldots,X_d)^\top\in\mathbb{R}^d$ denote
the vector of observed system variables, indexed by
$[d]=\{1,\ldots,d\}$. Let $Y=X_o$, with $o\in[d]$, denote
the designated outcome variable for which an anomaly has
been detected. For any $S\subseteq[d]$,
$\mathbf{X}_S=(X_j)_{j\in S}$ denotes the corresponding
subvector. We observe two independent datasets,
$\mathcal{D}_{\mathrm{obs}}
=\{\mathbf{X}_{\mathrm{obs}}^{(n)}\}_{n=1}^{N_{\mathrm{obs}}}$
and
$\mathcal{D}_{\mathrm{int}}
=\{\mathbf{X}_{\mathrm{int}}^{(n)}\}_{n=1}^{N_{\mathrm{int}}}$,
comprising i.i.d.\ samples from the joint distributions
$P^{\mathrm{obs}}$ and $P^{\mathrm{int}}$ of $\mathbf{X}$
under the normal and anomalous operating regimes, respectively.

\textbf{Normal regime.}
We assume that, under the normal operation regime, the system follows an acyclic Markovian Structural Causal Model (SCM) \citep{Pearl2009}:
\begin{equation}
    X_j
    =
    f_j^{\mathrm{obs}}
    \bigl(\mathbf{X}_{\mathrm{Pa}_G(j)},U_j\bigr),
    \qquad j\in[d],
\end{equation}
where $G$ is an unknown directed acyclic graph (DAG),
$\mathbf{X}_{\mathrm{Pa}_G(j)}$ denotes the vector of parent variables of $X_j$ in $G$, $U_j$ is an exogenous variable, and
$f_j^{\mathrm{obs}}$ is a measurable structural equation.
We assume causal sufficiency, with mutually independent exogenous variables $U_1,\ldots,U_d$.

\textbf{Anomalous regime.}
We model the anomalous regime as arising from a fixed intervention on
an unknown, nonempty target set $\mathcal{T}\subseteq[d]\setminus\{o\}$.
The intervention modifies the data-generating SCM  by substituting the mechanism of the target variable, or its exogenous noise distribution, or both, while leaving the mechanisms of all non-target variables unchanged.
Specifically, under the anomalous regime,
\begin{equation}
    X_j =
    \begin{cases}
        f_j^{\mathrm{int}}
        \bigl(\mathbf{X}_{\mathrm{Pa}_G(j)},U_j^{\mathrm{int}}\bigr),
        & j\in\mathcal{T},\\
        f_j^{\mathrm{obs}}
        \bigl(\mathbf{X}_{\mathrm{Pa}_G(j)},U_j^{\mathrm{obs}}\bigr),
        & j\notin\mathcal{T}.
    \end{cases}
\end{equation}
The exogenous noise variables remain mutually independent under the anomalous regime.
Non-target variables retain their original mechanisms and exogenous noise distributions, but their marginal distributions may still shift as intervention effects on ancestors propagate
through the SCM. This formulation permits simultaneous interventions on
different variables. We further assume that all data samples in $\mathcal{D}_{\mathrm{int}}$ arise from
the same interventions.

\textbf{Root-cause set and recovery objective.} We seek to identify the root causes of an anomaly detected
in the designated outcome variable $Y=X_o$. 
Let $\mathrm{An}_G(Y)$ denote the set of indices of the
strict ancestors of the outcome $Y$ in the normal-regime graph $G$ \footnote{Although direct interventions on the outcome $Y$ may occur in practice, we restrict our analysis to root causes among its strict ancestors.}.
We define the outcome-specific
root-cause set as
\begin{equation}
    \mathcal{R}_Y
    :=
    \mathcal{T}\cap\mathrm{An}_G(Y).
\end{equation}
We consider anomalous regimes in which the marginal
distribution of the outcome variable differs from that
under normal operation:
\begin{equation}
    P_Y^{\mathrm{obs}}
    \neq
    P_Y^{\mathrm{int}}.
\end{equation}
Under the interventional interpretation described above, this
marginal shift is induced by interventions on the
variables indexed by $\mathcal{R}_Y$, whose effects
propagate to the outcome through the system's causal
dependencies.

For the recovery problem studied here, we further assume
that every intervention target is a strict ancestor of
the outcome and that no intervention target is an
ancestor of another:
\begin{equation}
    \mathcal{T}\subseteq\mathrm{An}_G(Y),
    \qquad
    \mathcal{T}\cap\mathrm{An}_G(X_j)=\varnothing,
    \quad \text{for all } j\in\mathcal{T}.\label{eq:antichain_intervention}
\end{equation}
The second condition ensures that no directed path in $G$ contains more than one intervention target, thereby excluding cases in which one root cause lies causally
downstream of another.
Under these assumptions, $\mathcal{R}_Y=\mathcal{T}$.
Given the designated outcome $Y$ and the datasets
$\mathcal{D}_{\mathrm{obs}}$ and $\mathcal{D}_{\mathrm{int}}$,
the objective is to recover this root-cause set without
prior knowledge of the underlying causal graph $G$.

\begin{figure}[t]
    \centering
    \includegraphics[width=.87\linewidth]{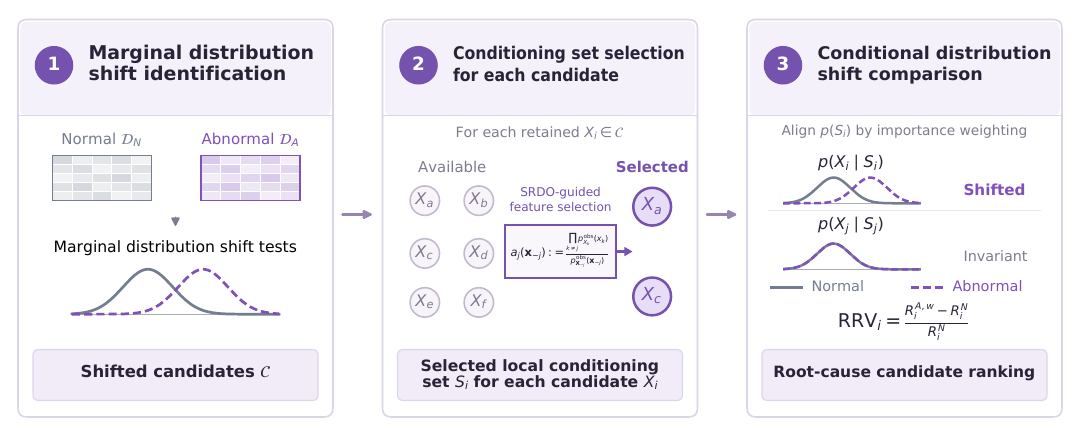}
    \caption{Overview of \ourmethod: (1) marginal distribution shift filtering, (2) local conditioning set selection, and (3) local conditional distribution comparison using RRV scores.}
    \label{fig:model}
\end{figure}

\section{Method}
Under the problem setting established in Section~\ref{sec:problem_formulation}, we present in this section a root-cause identification framework that combines marginal distribution shift filtering with conditional distribution shift detection. Subsection~\ref{subsec:general_method} develops the general framework and establishes sufficient conditions for root cause identification using suitable conditioning sets. Subsection~\ref{subsec:stablerca_practical_implementation} presents StableRCA, which uses stable learning to select these conditioning sets and compares prediction risks across the two regimes to assess changes in
conditional distributions. 
\subsection{Root-Cause Identification via Marginal Filtering and Local Conditional Shifts} \label{subsec:general_method}

An intervention can change not only the values of its target variables but also those of their descendants. Marginal distribution shifts and anomaly scores such as the $z$-score~\citep{rousseeuw2018anomaly} can therefore help identify abnormal variables, but do not by themselves reveal where the abnormality originated. When the causal graph is known, we can account for upstream changes by conditioning each variable on its parents and comparing $P(X_j\mid\mathbf{X}_{\mathrm{Pa}_G(j)})$ across regimes~\citep{budhathoki2021why}. When the graph is unknown, conditional distribution changes can be ambiguous: a parent variable that has not itself been intervened on may nevertheless exhibit a conditional distribution change when conditioned on an intervened child, even if its marginal distribution remains unchanged. Conditioning on all other variables does not necessarily resolve this ambiguity. We therefore first retain variables whose marginal distributions have shifted, excluding variables with unchanged marginals such as the parent in this example. We then compare the conditional distributions of the retained candidates across the two regimes given selected conditioning sets to distinguish root causes from variables affected only by propagation. Under appropriate conditions, this combination identifies root causes without recovering the complete causal graph.

We first define the set of candidate variables exhibiting marginal
distribution shifts:
\begin{equation}
    \mathcal{C}
    :=
    \left\{
    j\in[d]\setminus\{o\} :
    P_{X_j}^{\mathrm{obs}}
    \neq
    P_{X_j}^{\mathrm{int}}
    \right\}.
\end{equation}
For each candidate $j\in\mathcal{C}$, let
$S_j\subseteq[d]\setminus\{j\}$ denote a local conditioning set.
The conditional distribution shift selection rule retains candidates whose conditional distributions differ across regimes:
\begin{equation}
    \mathcal{R}_{\mathrm{cds}}
    :=
    \left\{
    j\in\mathcal{C} :
    P^{\mathrm{obs}}_{X_j\mid\mathbf{X}_{S_j}}
    \neq
    P^{\mathrm{int}}_{X_j\mid\mathbf{X}_{S_j}}
    \right\}.
\end{equation}
The question is then: \emph{Under which conditions can this rule
recover the true root-cause set, that is,
$\mathcal{R}_{\mathrm{cds}}=\mathcal{R}_Y$?}

Let $B_j$ denote a Markov boundary of $X_j$ under
$P^{\mathrm{obs}}$, that is, a minimal conditioning set
such that $X_j$ is independent of the remaining variables
given $\mathbf{X}_{B_j}$. If $P^{\mathrm{obs}}$ is Markov and faithful with respect to the
DAG $G$, the Markov boundary $B_j$ is unique and consists of the
parents, children, and spouses of $X_j$~\citep[Theorem~5]{statnikov2013algorithms}:
\begin{equation}
    B_j
    =
    \mathrm{Pa}_G(j)
    \cup
    \mathrm{Ch}_G(j)
    \cup
    \mathrm{Sp}_G(j),
\end{equation}
where $\mathrm{Ch}_G(j)$ denotes the children of $X_j$ and
$\mathrm{Sp}_G(j)$ denotes its spouses, namely the other
parents of its children.
We impose the following additional conditions on the two
regimes.

\begin{assumption}[Distributional regularity and marginal detectability]
\label{ass:population_identification}
\leavevmode
\begin{enumerate}
    \item The normal-regime distribution $P^{\mathrm{obs}}$
    is faithful to $G$: every conditional independence in
    $P^{\mathrm{obs}}$ is implied by a d-separation in $G$~\citep[Definition~6.33]{peters2017elements}.

    \item The distributions $P^{\mathrm{obs}}$ and $P^{\mathrm{int}}$
    admit strictly positive densities on the same Cartesian
    product support $\mathcal{X} = \prod_{k=1}^{d} \mathcal{X}_k.$

    \item Every intervention target exhibits a marginal
    distribution shift:
    \[
        P_{X_j}^{\mathrm{obs}} \neq P_{X_j}^{\mathrm{int}},
        \qquad j \in \mathcal{T}.
    \]
\end{enumerate}
\end{assumption}

Throughout, a conditional distribution shift means that
the two conditional laws differ on a set of conditioning
values with positive
$P_{\mathbf{X}_{S_j}}^{\mathrm{obs}}$-probability.
The common-support condition ensures that both
conditional laws are defined on the same domain.

\begin{proposition}[Population root-cause identification]
\label{prop:population_identification}
Under the SCM and intervention assumptions stated in the
problem formulation and
Assumption~\ref{ass:population_identification}, suppose
that, for every $j\in\mathcal{C}$, the conditioning set
satisfies
\[
    B_j\subseteq S_j\subseteq[d]\setminus\{j\}.
\]
Then
\begin{equation}
    \mathcal{R}_{\mathrm{cds}}
    =
    \mathcal{T}
    =
    \mathcal{R}_Y.
\end{equation}
\end{proposition}

\emph{Proof sketch.}
Every marginally shifted variable is an intervention
target or a descendant of an intervention target.
Together with the assumption that no intervention target is an ancestor
of another, as stated in Equation~\ref{eq:antichain_intervention},
this implies that no candidate in $\mathcal{C}$ has an
intervention target as a child.
For a non-target variable in $\mathcal{C}$, its own mechanism and the
mechanisms of all its children therefore remain unchanged.
Since these mechanisms determine its full conditional
distribution, its conditional distribution remains
invariant when conditioning on $B_j$ or any superset. For an intervention target, the joint distribution of its parents is unchanged.
Its marginal shift therefore implies a change in its
conditional mechanism. Strict positivity ensures that
this change remains detectable after conditioning on
its Markov boundary or any superset.
A complete proof is provided in Appendix~\ref{app:population_proof}.

Proposition~\ref{prop:population_identification} establishes
sufficient conditions for identifying root causes from marginal
and conditional distribution shifts without recovering the causal graph.
Under the stated assumptions, the Markov boundary $B_j$
provides a minimal set that preserves the full conditional distribution. Any conditioning set $S_j$ containing $B_j$ also suffices, so exact boundary recovery is unnecessary.
Although this permits conditioning on all other variables,
including many redundant predictors increases the
estimation burden and, with limited data or model
misspecification, may lead fitted models to exploit
spurious associations.
This motivates the use of compact local conditioning
sets that preserve sufficient conditional information
while limiting unnecessary predictors.
In practice, the local conditioning sets are unknown, and conditional distribution shifts must be assessed empirically. 
Subsection~\ref{subsec:stablerca_practical_implementation} presents
StableRCA, which uses independence-driven sample reweighting to
select conditioning sets and compares prediction risks across
regimes to assess changes in conditional distributions.

\subsection{StableRCA: Practical Implementation} \label{subsec:stablerca_practical_implementation}
StableRCA implements the framework in Section~\ref{subsec:general_method}
in three steps: (1) identifying candidates with marginal distribution
shifts, (2) selecting a conditioning set for each candidate, and
(3) scoring changes in their conditional distributions.
Figure~\ref{fig:model} illustrates the procedure.

\textbf{Identifying candidates with marginal distribution
shifts.}
For each $j\in[d]\setminus\{o\}$, we test the null hypothesis
$P_{X_j}^{\mathrm{obs}}=P_{X_j}^{\mathrm{int}}$ using
$\mathcal{D}_{\mathrm{obs}}$ and $\mathcal{D}_{\mathrm{int}}$.
We apply the two-sample Kolmogorov--Smirnov test to
continuous variables~\citep{conover1999practical} and the
$\chi^2$ test to discrete variables~\citep{pearsonchisquared}.
Writing $\pi_j$ for the resulting $p$-value, we define the
empirical candidate set as
\begin{equation}
    \widehat{\mathcal{C}}
    :=
    \left\{
    j\in[d]\setminus\{o\}:\pi_j\leq\alpha
    \right\},
    \label{eq:empirical_screening}
\end{equation}
where $\alpha$ is the per-variable significance level.
This step narrows the set of potential root causes considered in subsequent steps to variables with evidence of a marginal distribution shift.

\textbf{Selecting a conditioning set for each candidate.}

To construct the conditioning sets, we use stable learning, which was
originally developed to improve prediction under distribution shifts
by reducing reliance on spurious predictive associations through
causally motivated regularization. For a candidate $X_j$, consider the
smallest set of predictors that yields the same conditional mean as
using all other variables. Prior work shows that, under strict
positivity, this set is contained in the Markov boundary of
$X_j$~\citep{pmlr-v162-xu22o}. If the conditional mean also
determines the full conditional distribution of $X_j$, the two sets
coincide, and any predictor set that preserves the conditional mean
contains the Markov boundary. Motivated by this connection, we use
the Sample Reweighted Decorrelation Operator
(SRDO)~\citep{shen2020stable}, a stable learning method, to guide
the selection of a conditioning set for each candidate, with the
aim of reducing the inclusion of redundant or spuriously associated
variables.

In practice, we first split the normal data into a training set
$\mathcal{D}_{\mathrm{obs}}^{\mathrm{tr}}$ and an evaluation
set $\mathcal{D}_{\mathrm{obs}}^{\mathrm{te}}$.
For each candidate $j\in\widehat{\mathcal{C}}$, we treat $X_j$ as the
prediction target and
$\mathbf{X}_{-j}:=\mathbf{X}_{[d]\setminus\{j\}}$ as the
initial predictors.
Using the normal training data, SRDO estimates sample weights
targeting the ratio
\begin{equation}
    a_j(\mathbf{x}_{-j})
    :=
    \frac{
        \prod_{k\neq j}p_{X_k}^{\mathrm{obs}}(x_k)
    }{
        p_{\mathbf{X}_{-j}}^{\mathrm{obs}}(\mathbf{x}_{-j})
    },
    \label{eq:independence_weights}
\end{equation}
where the densities are defined with respect to appropriate
reference measures for continuous and discrete variables. SRDO
estimates this ratio by contrasting the observed predictors with
independently resampled versions.

Under exact reweighting, predictors outside the Markov boundary of
$X_j$ are independent of both $X_j$ and its boundary variables.
They can therefore be excluded without changing the conditional
distribution of $X_j$ in the reweighted population. Motivated by this property, we use the estimated
weights to fit a regression or classification model for $X_j$,
depending on its type. We then apply a relative threshold to the
model's feature importances to select the conditioning set
$\widehat S_j$ as a practical proxy for the Markov boundary. Appendix A.2 provides further details on the reweighting, and
Appendix B specifies the models and selection rule. We assess the
contribution of these components through ablation studies in
Section 5.4.

\textbf{Conditional distribution shift scoring.}
We use changes in prediction risk as a practical proxy for
conditional distribution shifts. For each candidate $j$, we fit a
prediction model $\widehat f_j$ on the normal training data
$\mathcal{D}_{\mathrm{obs}}^{\mathrm{tr}}$ to predict $X_j$ from its
selected conditioning variables $\mathbf{X}_{\widehat S_j}$.
We then evaluate the same fitted model on the held-out normal data
$\mathcal{D}_{\mathrm{obs}}^{\mathrm{te}}$ and the anomalous data
$\mathcal{D}_{\mathrm{int}}$, using mean squared error for continuous
$X_j$ and cross-entropy for discrete $X_j$. However, the resulting risks may differ because the distribution of
the conditioning variables $\mathbf{X}_{\widehat S_j}$ can change
as the effects of the anomaly propagate through the system, even when the conditional distribution of $X_j$ given these variables remains
unchanged. This situation is known as covariate
shift~\citep{Shimodaira2000}.

To account for this issue, we apply importance
weighting~\citep{Shimodaira2000} by multiplying each anomalous
sample's prediction loss by a weight that reflects how common its
conditioning-variable values are in the normal regime relative to
the anomalous regime. This reweighting aims to match the distribution
of the conditioning variables used to compute the anomalous risk to
their distribution in the normal regime.\footnote{These weights
differ from the SRDO weights $a_j$, which are used to reduce
dependencies among predictors during conditioning-set selection.} Under the common-support condition in
Subsection~\ref{subsec:general_method}, we define the weight for the
conditioning variables as
\begin{equation}
    w_j(\mathbf{X}_{\widehat S_j})
    :=
    \frac{
        p_{\mathbf{X}_{\widehat S_j}}^{\mathrm{obs}}
        (\mathbf{X}_{\widehat S_j})
    }{
        p_{\mathbf{X}_{\widehat S_j}}^{\mathrm{int}}
        (\mathbf{X}_{\widehat S_j})
    }.
    \label{eq:alignment_weights}
\end{equation}
For each anomalous sample, we evaluate this ratio at its observed
conditioning-variable values and multiply its prediction loss by
the resulting weight. With exact weights, the weighted anomalous risk
and the normal risk are averaged over the same distribution of
conditioning variables.

For a fixed conditioning set and fitted predictor, let $\ell_j$
denote the prediction loss for $X_j$. We define the normal and
weighted anomalous risks as
\begin{align}
    R_j^{\mathrm{obs}}
    &:=
    \mathbb{E}_{P^{\mathrm{obs}}}
    \left[
        \ell_j\bigl(
        \widehat f_j(\mathbf{X}_{\widehat S_j}),X_j
        \bigr)
    \right],
    \\
    R_j^{\mathrm{int,align}}
    &:=
    \mathbb{E}_{P^{\mathrm{int}}}
    \left[
        w_j(\mathbf{X}_{\widehat S_j})
        \ell_j\bigl(
        \widehat f_j(\mathbf{X}_{\widehat S_j}),X_j
        \bigr)
    \right].
    \label{eq:aligned_population_risk}
\end{align}
With exact weights and finite risks, the two risks are equal whenever
the conditional distribution of $X_j$ given
$\mathbf{X}_{\widehat S_j}$ is unchanged.

We estimate the density ratio $w_j$ from the selected conditioning
variables in the normal and anomalous datasets, as described in
Appendix~\ref{appendix:stablerca_implementation}. Let $\widehat w_j$
denote the estimated ratio, and write
$N_{\mathrm{te}}=|\mathcal{D}_{\mathrm{obs}}^{\mathrm{te}}|$
and $N_{\mathrm{int}}=|\mathcal{D}_{\mathrm{int}}|$.
The corresponding empirical risks are
\begin{align}
    \widehat R_j^{\mathrm{obs}}
    &:=
    \frac{1}{N_{\mathrm{te}}}
    \sum_{\mathbf{x}\in\mathcal{D}_{\mathrm{obs}}^{\mathrm{te}}}
    \ell_j\bigl(
        \widehat f_j(\mathbf{x}_{\widehat S_j}),x_j
    \bigr),
    \\
    \widehat R_j^{\mathrm{int,align}}
    &:=
    \frac{1}{N_{\mathrm{int}}}
    \sum_{\mathbf{x}\in\mathcal{D}_{\mathrm{int}}}
    \widehat w_j(\mathbf{x}_{\widehat S_j})
    \ell_j\bigl(
        \widehat f_j(\mathbf{x}_{\widehat S_j}),x_j
    \bigr).
    \label{eq:empirical_aligned_risks}
\end{align}
Their difference estimates the change in prediction risk after
adjusting for differences in the distribution of the conditioning
variables.

Because prediction risks can vary in scale across candidates, we
normalize each candidate's risk difference by its normal-regime
risk. For $\widehat R_j^{\mathrm{obs}}>0$, we define the
\emph{Relative Risk Variation} (RRV) score as
\begin{equation}
    \widehat s_j
    :=
    \frac{
        \widehat R_j^{\mathrm{int,align}}
        -
        \widehat R_j^{\mathrm{obs}}
    }{
        \widehat R_j^{\mathrm{obs}}
    },
    \qquad j\in\widehat{\mathcal{C}}.
    \label{eq:rca_score}
\end{equation}
This dimensionless score measures the proportional
change in prediction risk and is invariant to
multiplicative rescaling of the loss.
We rank candidates in descending order of $\widehat s_j$,
treating larger relative risk increases as stronger
root-cause evidence.

RRV uses changes in prediction risk after importance weighting as a proxy for changes in the conditional distribution of $X_j$ given its selected local conditioning set. With exact weights, finite risks, and a positive normal-regime risk, the score is zero when this conditional distribution remains unchanged. When the conditional distribution changes, the RRV value depends on the fitted predictor and the choice of prediction loss. We therefore evaluate the practical effectiveness of RRV empirically through experiments on both synthetic data and real-world application scenarios in Section~\ref{sec:experiments}, and use oracle diagnostics in
Appendix~\ref{app:rrv_diagnostic} to examine RRV under
variance-decreasing and equal-risk interventions, where changes in the conditional distribution do not necessarily increase prediction risk.

\section{Experiments}\label{sec:experiments}

\subsection{Experimental setup}
We evaluate \ourmethod\ through three complementary studies. First, we compare it with baseline methods on synthetic data generated from known SCMs, considering different system sizes and scenarios with single or multiple root causes to assess its ability to identify root causes under controlled interventions. Second, we assess its practical effectiveness and robustness on real-world and semi-synthetic application benchmarks spanning microservice systems, manufacturing, and simulated physical systems. Finally, we conduct ablation studies to examine how SRDO-based conditioning set selection and importance-weighted RRV scoring contribute to distinguishing root causes from propagated effects. Additional results on scalability to large systems, sensitivity to the number of anomalous samples, and comparisons with graph-based methods under graph corruption are provided in Appendices~\ref{appendix:scalability_large_system}, \ref{appendix:sensitivity_abnormal_sample}, and~\ref{appendix:graph_corruption}, respectively.

\textbf{Synthetic datasets.} We generate data from SCMs defined on random DAGs sampled using the Erdős–Rényi (ER) model \citep{erdos59a} with $d$ variables and $2d$ edges. We generate each non-root variable using an additive noise model, with each parent’s contribution modeled by either a linear function or an MLP. Anomalies are introduced through soft-function interventions
that resample the incoming edge functions of selected
non-root variables, while leaving the graph, noise
distributions, and all non-target mechanisms unchanged.
The replacement functions remain fixed across anomalous
samples. Intervention targets are selected from the strict ancestors
of a common designated outcome variable, with no target
being an ancestor of another, following the structural
restrictions in Section~\ref{sec:problem_formulation}.
Our main experiments consider single-target and five-target
interventions, with system-size comparisons at
$d\in\{40,60,80\}$.
Each configuration is evaluated over three runs, each
comprising $20$ independently generated SCMs with
$2{,}000$ normal and $200$ anomalous observations per SCM.
Within each trial, all methods use the same datasets,
intervention targets, and designated outcome.
We report the mean and standard deviation of the
run-level performance.
Detailed generation procedures and configurations are
provided in Appendix~\ref{appendix:dataset_description}.

\textbf{Application benchmarks.} We evaluate \ourmethod\ on four real-world application scenario benchmarks spanning retail, microservices, manufacturing, and physical systems. \textit{ProRCA}~\citep{dawoud2025prorcacausalpythonpackage} provides simulated retail transaction data with controlled anomaly injections that propagate through business variables to outcomes. \textit{RCAEval}~\citep{pham2025rcaeval} is a collection of RCA datasets containing telemetry data from three microservice systems: Online Boutique, Sock Shop, and Train Ticket. The data include CPU usage, memory usage, response latency, and other monitoring metrics, along with annotated root causes. \textit{CausalMan}~\citep{tagliapietra2025causalmanphysicsbasedsimulatorlargescale} provides observational and interventional data from a large-scale physics-based manufacturing simulator, featuring complex nonlinear dependencies and mixed continuous, categorical, and binary variables.  \textit{CausalChambers}~\citep{gamella2025chamber} contains measurements from real light- and wind-tunnel experiments, with controlled interventions and empirically validated causal models.

Application performance is evaluated at the variable level for ProRCA, CausalMan, and CausalChambers, and at the service level for RCAEval, where the service containing the top-ranked metric is compared with the annotated root-cause service. Graph-based baselines are provided with the ground-truth causal graphs for ProRCA, CausalMan, and CausalChambers, and a graph estimated from normal data using XGES~\citep{nazaret2021extremely} for RCAEval since the true dependency graphs are not available. Dataset statistics and evaluation configurations are detailed in Appendix~\ref{appendix:dataset_description}.

\textbf{Baselines.}
We compare \ourmethod\ with ten representative RCA baselines: \textbf{graph-based methods}, including Traversal~\citep{6848128}, CIRCA~\citep{10.1145/3534678.3539041}, Smooth Traversal~\citep{orchard2025root}, and Counterfactual Attribution~\citep{budhathoki2022causal}; \textbf{causal methods that do not require a supplied graph}, including RCD~\citep{10.5555/3600270.3602529}, Cholesky~\citep{10.1093/jrsssb/qkaf066}, Score Ordering~\citep{orchard2025root}, and RCG-0~\citep{ikram2025partialrca}; and \textbf{statistical methods}, including $\epsilon$-Diagnosis~\citep{10.1145/3308558.3313653} and BARO~\citep{pham2024baro}. Baseline configurations and implementation details are provided in Appendix~\ref{appendix:baseline_models}.

\textbf{Metrics.}
We report Top-1 accuracy for cases with a single root cause and precision and recall at $k$ for cases with multiple root causes, where $k$ denotes the number of true root causes. Precision and recall coincide when exactly $k$ distinct variables are selected. For methods that score candidate causes separately for each anomalous sample, we average each variable's scores across all anomalous samples to obtain a single root-cause ranking per RCA case.

\subsection{Evaluation on synthetic data}\label{sec:res_syn}
\textbf{Results with a single root cause.} Table~\ref{tab:synthetic_graph_size_root_causes} shows that \ourmethod\ achieves the highest mean Top-1 accuracy across all three graph sizes in the single-root-cause setting. Its accuracies are $0.83$, $0.87$, and $0.80$ for $n=40$, $60$, and $80$, respectively, indicating consistently strong performance over the tested system sizes. Traversal and Smooth Traversal obtain lower mean accuracies
despite using ground-truth graphs, while BARO and
$\epsilon$-Diagnosis reach accuracies of at most $0.22$
and $0.02$, respectively.
Overall, these results suggest that \ourmethod\ provides
competitive single-root-cause identification across the
tested system sizes without requiring a supplied causal graph.

\textbf{Results with multiple simultaneous root causes.}
With five simultaneous interventions, \ourmethod\ achieves precision/recall@$5$ of $0.61$, $0.58$, and $0.58$ for $n=40$, $60$, and $80$, respectively, maintaining strong performance across the tested system sizes. It attains the highest reported mean among graph-free methods in all three settings, tying Score Ordering at $n=60$. Its small standard deviations of $0.02$--$0.03$ further indicate limited variability across repeated runs. Traversal achieves higher scores when supplied with ground-truth graphs, but its performance decreases to $0.51$, $0.55$, and $0.56$ when XGES-discovered graphs are used. This decrease suggests that graph-estimation errors may affect subsequent root-cause ranking, making the quality of the supplied structure an important factor in the overall diagnostic pipeline.  Across all three sizes, \ourmethod\ outperforms every reported XGES-based baseline, supporting local conditional comparisons as an effective alternative to the two-stage approach of reconstructing a global causal graph and then identifying intervention targets. These findings are consistent with our theoretical identification principle: under the stated assumptions, marginal distribution filtering followed by suitable local conditional comparisons can distinguish multiple intervention targets from variables exhibiting propagated distribution shifts.

\newcommand{\rcatablefont}{\footnotesize}
\newcommand{\rcatablesetup}{%
\setlength{\tabcolsep}{4pt}
\renewcommand{\arraystretch}{0.92}
\setlength{\aboverulesep}{1pt}
\setlength{\belowrulesep}{1pt}
\setlength{\abovetopsep}{0pt}
\setlength{\belowcaptionskip}{2pt}
}


\begin{table}[t]
\centering
\rcatablesetup
\caption{Results on synthetic data with varying graph sizes ($n$)
and numbers of simultaneous root causes.
For graph-based methods, results outside parentheses use
ground-truth graphs, while results in parentheses use
XGES-discovered graphs.
Best reported means are bold separately for graph-based and graph-free methods.}
\label{tab:synthetic_graph_size_root_causes}

{\rcatablefont
\setlength{\tabcolsep}{4pt}
\renewcommand{\arraystretch}{1.15}

\resizebox{\textwidth}{!}{%
\begin{tabular}{p{1.7cm}p{2.8cm}l*{6}{c}}
\toprule
& & & \multicolumn{3}{c}{\textbf{1 root cause}}
    & \multicolumn{3}{c}{\textbf{5 simultaneous root causes}} \\
& & & \multicolumn{3}{c}{Top-1 accuracy}
    & \multicolumn{3}{c}{Precision/Recall@$5$} \\
\cmidrule(lr){4-6}\cmidrule(lr){7-9}
\textbf{Category} & \textbf{Subcategory} & \textbf{Method}
& \textbf{$n=40$} & \textbf{$n=60$} & \textbf{$n=80$}
& \textbf{$n=40$} & \textbf{$n=60$} & \textbf{$n=80$} \\
\midrule

\multirow{8}{1.7cm}{%
    \raggedright
    \textit{Graph-based methods}\par
    \vspace{3pt}
    {\scriptsize
    Outside $(\,)$: true graph\par
    \vspace{2pt}
    Inside $(\,)$: XGES graph}
}
& \multirow{8}{2.8cm}{\centering---}
& Traversal
& $\mathbf{0.73{\pm}0.14}$
& $\mathbf{0.77{\pm}0.13}$
& $\mathbf{0.78{\pm}0.10}$
& $\mathbf{0.69{\pm}0.15}$
& $\mathbf{0.68{\pm}0.19}$
& $\mathbf{0.72{\pm}0.18}$ \\
& & &
$(0.57{\pm}0.23)$ &
$(0.60{\pm}0.10)$ &
$(0.65{\pm}0.22)$ &
$(0.51{\pm}0.22)$ &
$(0.55{\pm}0.22)$ &
$(0.56{\pm}0.23)$ \\

& & Smooth Traversal
& $0.65{\pm}0.15$ & $0.72{\pm}0.08$ & $0.73{\pm}0.08$
& $0.61{\pm}0.18$ & $0.63{\pm}0.19$ & $0.66{\pm}0.16$ \\
& & &
$(0.53{\pm}0.28)$ &
$(0.60{\pm}0.10)$ &
$(0.60{\pm}0.13)$ &
$(0.46{\pm}0.24)$ &
$(0.48{\pm}0.21)$ &
$(0.54{\pm}0.23)$ \\

& & CIRCA
& $0.63{\pm}0.08$ & $0.62{\pm}0.03$ & $0.57{\pm}0.16$
& $0.29{\pm}0.02$ & $0.30{\pm}0.04$ & $0.27{\pm}0.04$ \\
& & &
$(0.53{\pm}0.03)$ &
$(0.48{\pm}0.08)$ &
$(0.53{\pm}0.10)$ &
$(0.34{\pm}0.16)$ &
$(0.36{\pm}0.19)$ &
$(0.32{\pm}0.20)$ \\

& & Count. Attr.
& $0.20{\pm}0.09$ & $0.23{\pm}0.08$ & $0.25{\pm}0.05$
& $0.33{\pm}0.20$ & $0.28{\pm}0.23$ & $0.22{\pm}0.22$ \\
& & &
$(0.15{\pm}0.05)$ &
$(0.17{\pm}0.06)$ &
$(0.22{\pm}0.08)$ &
$(0.26{\pm}0.21)$ &
$(0.26{\pm}0.20)$ &
$(0.16{\pm}0.18)$ \\
\midrule

\multirow{7}{1.7cm}{
    \raggedright\textit{Graph-free methods}
}
& \multirow{2}{2.8cm}{
    \raggedright\textit{Statistical methods}
}
& $\epsilon$-Diag.
& $0.02{\pm}0.03$ & $0.00{\pm}0.00$ & $0.02{\pm}0.03$
& $0.03{\pm}0.02$ & $0.03{\pm}0.03$ & $0.05{\pm}0.02$ \\

& & BARO
& $0.17{\pm}0.03$ & $0.15{\pm}0.05$ & $0.22{\pm}0.03$
& $0.30{\pm}0.02$ & $0.31{\pm}0.02$ & $0.27{\pm}0.02$ \\
\cmidrule(lr){2-9}

& \multirow{5}{2.8cm}{
    \raggedright\textit{Causal methods without a supplied graph}
}
& RCD
& $0.80{\pm}0.18$ & $0.82{\pm}0.06$ & $0.78{\pm}0.08$
& $0.39{\pm}0.24$ & $0.46{\pm}0.32$ & $0.40{\pm}0.25$ \\

& & Cholesky
& $0.52{\pm}0.06$ & $0.52{\pm}0.15$ & $0.57{\pm}0.06$
& $0.49{\pm}0.01$ & $0.47{\pm}0.04$ & $0.43{\pm}0.03$ \\

& & Score Ordering
& $0.60{\pm}0.09$ & $0.72{\pm}0.08$ & $0.72{\pm}0.16$
& $0.57{\pm}0.04$
& $\mathbf{0.58{\pm}0.02}$
& $0.54{\pm}0.02$ \\

& & RCG-0
& $0.17{\pm}0.03$ & $0.22{\pm}0.03$ & $0.15{\pm}0.09$
& $0.18{\pm}0.03$ & $0.17{\pm}0.02$ & $0.20{\pm}0.02$ \\
\cmidrule(lr){3-9}

& & \textbf{\ourmethod}
& $\mathbf{0.83{\pm}0.08}$
& $\mathbf{0.87{\pm}0.10}$
& $\mathbf{0.80{\pm}0.05}$
& $\mathbf{0.61{\pm}0.02}$
& $\mathbf{0.58{\pm}0.03}$
& $\mathbf{0.58{\pm}0.02}$ \\
\bottomrule
\end{tabular}%
}
}
\end{table}



\begin{table}[t]
\centering
\caption{Top-1 accuracy on four real-world RCA benchmarks.}
\label{tab:real_world_precision_std}

\resizebox{\linewidth}{!}{%
\begin{tabular}{p{1.7cm}p{2.8cm}lccccc}
\toprule
\textbf{Category}
& \textbf{Subcategory}
& \textbf{Method}
& \textbf{ProRCA}
& \textbf{CausalChamber}
& \textbf{CausalMan}
& \textbf{RCAEval}
& \textbf{Macro-Average} \\
\midrule

\multirow{4}{1.7cm}{
    \raggedright\textit{Graph-based methods}
}
& \multirow{4}{2.8cm}{\centering---}
& Traversal
& $0.75{\pm}0.43$ & $\mathbf{1.00{\pm}0.00}$
& $0.34{\pm}0.48$ & $0.09{\pm}0.29$
& $0.55$ \\

& & Smooth Traversal
& $0.75{\pm}0.43$ & $\mathbf{1.00{\pm}0.00}$
& $\mathbf{0.43{\pm}0.50}$ & $0.30{\pm}0.46$
& $0.62$ \\

& & CIRCA
& $0.25{\pm}0.43$ & $0.16{\pm}0.37$
& $0.00{\pm}0.00$ & $\mathbf{0.69{\pm}0.46}$
& $0.28$ \\

& & Counterfactual Contribution
& $0.50{\pm}0.50$ & $0.26{\pm}0.44$
& $0.11{\pm}0.32$ & $0.08{\pm}0.27$
& $0.24$ \\
\midrule

\multirow{7}{1.7cm}{
    \raggedright\textit{Graph-free methods}
}
& \multirow{2}{2.8cm}{
    \raggedright\textit{Statistical methods}
}
& $\epsilon$-Diagnosis
& $0.05{\pm}0.22$ & $0.00{\pm}0.00$
& $0.00{\pm}0.00$ & $0.12{\pm}0.33$
& $0.04$ \\

& & BARO
& $0.55{\pm}0.50$ & $0.08{\pm}0.26$
& $0.00{\pm}0.00$ & $0.61{\pm}0.49$
& $0.31$ \\
\cmidrule(lr){2-8}

& \multirow{5}{2.8cm}{
    \raggedright\textit{Causal methods without a supplied graph}
}
& RCD
& $0.70{\pm}0.46$ & $0.78{\pm}0.42$
& $0.41{\pm}0.49$ & $0.50{\pm}0.50$
& $0.60$ \\

& & Cholesky
& $0.80{\pm}0.40$ & $0.12{\pm}0.33$
& $0.00{\pm}0.00$ & $0.67{\pm}0.47$
& $0.40$ \\

& & Score Ordering
& $0.55{\pm}0.50$ & $\mathbf{1.00{\pm}0.00}$
& $0.29{\pm}0.45$ & $0.53{\pm}0.50$
& $0.59$ \\

& & RCG-0
& $0.50{\pm}0.50$ & $\mathbf{1.00{\pm}0.00}$
& $0.00{\pm}0.00$ & $0.40{\pm}0.49$
& $0.48$ \\
\cmidrule(lr){3-8}

& & \textbf{\ourmethod}
& $\mathbf{1.00{\pm}0.00}$ & $0.76{\pm}0.43$
& $0.40{\pm}0.49$ & $0.66{\pm}0.47$
& $\mathbf{0.71}$ \\
\bottomrule
\end{tabular}%
}
\end{table}

\begin{table}[t]
\centering
\rcatablesetup
\caption{Ablation study on a graph with 50 variables and 100 edges
under five simultaneous interventions.}
\label{tab:ablation_study}

{\rcatablefont
\setlength{\tabcolsep}{5pt}
\renewcommand{\arraystretch}{1.15}

\begin{tabular}{@{}lc@{}}
\toprule
\textbf{Variant} & \textbf{Precision/Recall@$5$} \\
\midrule
Oracle MB                                     & $\mathbf{0.59{\pm}0.04}$ \\
\midrule
\textbf{\ourmethod}                           & $\mathbf{0.58{\pm}0.02}$ \\
Marginal p-value                              & $0.52{\pm}0.02$ \\
Condition on all remaining variables          & $0.55{\pm}0.04$ \\
Corrupted MB (25\%)                           & $0.57{\pm}0.02$ \\
Corrupted MB (50\%)                           & $0.54{\pm}0.00$ \\
CatBoost without SRDO                         & $0.55{\pm}0.03$ \\
LASSO                                         & $0.55{\pm}0.03$ \\
Compute RRV score without covariate alignment & $0.54{\pm}0.02$ \\
\bottomrule
\end{tabular}
}
\end{table}

\subsection{Evaluation on real-world application benchmarks}
Table~\ref{tab:real_world_precision_std} reports results across four real-world application benchmarks. \ourmethod\ achieves the highest reported macro-average
Top-1 accuracy of $0.71$, compared with $0.62$ for
Smooth Traversal.
This macro-average assigns equal weight to each benchmark
and summarizes performance over the selected datasets.
Graph-based baselines receive ground-truth graphs on
three benchmarks, whereas \ourmethod\ does not require
a supplied causal graph.
The per-dataset results nevertheless show substantial
variation in the relative performance of the methods.

On CausalChamber, intervention targets typically influence the observed outcomes through one-step causal paths, and interventions move the target variables outside their normal ranges, inducing large marginal distribution shifts. These characteristics may help explain the perfect
accuracy of Score Ordering and traversal-based methods,
which exploit marginal anomaly scores.
For \ourmethod, the detailed results in
Appendix~\ref{appendix:more_exp_results} show that all
color and LED interventions are correctly identified,
with errors concentrated in polarizer interventions.
Interventions outside normal operating ranges may reduce
support overlap in the selected conditioning variables,
making covariate-aligned risk estimation less reliable.
The periodic dependence of light intensity on relative
polarizer angles may introduce additional prediction
challenges.

On CausalMan, which features hundreds of variables,
mixed variable types, and nonlinear physical mechanisms,
\ourmethod\ achieves an accuracy of $0.40$, compared with
$0.43$ for Smooth Traversal using the ground-truth graph.
The benchmark includes both hard and soft interventions,
providing an empirical assessment across different
mechanism changes. On RCAEval, we use the metric data from RE1 and RE2 for evaluation. 
These datasets cover resource and network faults, whereas RE3 is
designed for multimodal RCA of code-level faults and is therefore
not suited to our setting.
Because ground-truth causal graphs are unavailable,
graph-based baselines use graphs estimated from normal
data using XGES~\citep{nazaret2021extremely}.
\ourmethod\ achieves an accuracy of $0.66$, compared with
$0.69$ for CIRCA and $0.67$ for Cholesky, and obtains
higher reported accuracy than the remaining baselines. 

These application experiments complement the synthetic
evaluations by assessing the practical performance of
\ourmethod\ in settings where the assumptions of
Proposition~4.2 may not hold. Deterministic hard interventions on continuous variables
or changes in support may violate the common-support and
strict positivity conditions.
Limited overlap can destabilize importance-weighted
estimates, while a lack of support coverage prevents
exact alignment. Nevertheless, \ourmethod\ remains competitive on benchmarks
that include such assumption violations, suggesting
empirical robustness to some departures from the
conditions required by the theoretical analysis.

\subsection{Ablation studies}\label{subsec:ablation_studies}
We conduct ablation studies to assess the contributions of conditional
shift scoring, local conditioning set selection, and importance weighting of the RRV score
to root-cause identification. All variants are evaluated on a graph
with 50 variables and 100 edges under five simultaneous interventions.
Table~\ref{tab:ablation_study} reports precision/recall@$5$ as the
mean $\pm$ population standard deviation of three run means, each
computed over 20 RCA cases.

\textbf{Conditional versus Marginal Shift Scoring.}
To assess the benefit of considering conditional distribution shifts
in addition to marginal shifts, we compare the full method with a
variant that ranks candidates in ascending order of the $p$-values
from their marginal distribution shift tests. \ourmethod\ achieves precision/recall@$5$ of $0.58$, compared with $0.52$ for the marginal variant, corresponding to an $11.5\%$ relative improvement. Marginal shifts can occur at both intervention targets and their descendants, so statistical evidence of marginal change alone may not reliably distinguish root causes from propagated effects. The observed improvement supports the advantage of combining marginal filtering with conditional shift scoring to identify root causes under multiple simultaneous interventions.

\textbf{Choice of Local Conditioning Sets.}
We examine how different conditioning-set selection strategies affect root-cause identification. The full method achieves precision/recall@5 of $0.58$, close to the oracle-MB reference of $0.59$, demonstrating competitive root-cause
identification performance with conditioning sets learned using
causal regularization. Corrupting the oracle boundaries reduces performance to $0.57$ and $0.54$ at corruption rates of $25\%$ and $50\%$, respectively, illustrating the value of retaining informative conditioning variables. Conditioning on all remaining variables yields $0.55$, consistent
with the increased difficulty of fitting prediction models and
aligning the covariate distributions of many conditioning variables. CatBoost selection without SRDO and LASSO selection both achieve $0.55$, compared with $0.58$ for SRDO-guided selection. The observed improvement over unweighted CatBoost provides empirical support for SRDO reweighting in the evaluated configuration. These results support the selection of conditioning variables using
causal regularization as a practical component of \ourmethod\
when the underlying causal structure is unknown.

\textbf{Importance weighting in RRV Scoring.}
Finally, we examine the contribution of importance weighting for covariate alignment in RRV scoring. Specifically, we compare the full method with a variant that computes RRV without the weights in
Equation~\ref{eq:alignment_weights}, while retaining the same
marginal filtering and conditioning set selection procedures. The full method achieves precision/recall@$5$ of $0.58$, compared
with $0.54$ for the unweighted variant, corresponding to an $7.4\%$ relative improvement. Without weighting, differences in the distributions of the
conditioning variables across regimes can affect prediction risks
even when the conditional distribution of the candidate remains unchanged. Importance weighting accounts for these differences by
weighting anomalous samples according to how common their
conditioning-variable values are in the normal regime relative to
the anomalous regime. The results are consistent with this adjustment
helping RRV scores reflect changes in conditional distributions.

\section{Conclusions}\label{sec:conclusions}
In this work, we presented \ourmethod, a framework for
mechanism-level root-cause analysis without global causal graph
reconstruction. The framework identifies candidates based on
marginal distribution shifts, selects conditioning variables for
each candidate using SRDO, and ranks the candidates using
importance-weighted RRV scores as a practical proxy for changes in
their conditional distributions. Our theoretical analysis establishes
sufficient conditions under which marginal filtering and conditional
distribution comparisons distinguish intervention targets from
variables affected by propagation at the population level.
Experiments on synthetic data demonstrate the effectiveness of
\ourmethod\ under both single and multiple interventions. Across
four real-world application benchmarks, \ourmethod\ achieves the
highest average Top-1 accuracy among the evaluated methods. Ablation
studies further support the contributions of conditional shift
scoring, conditioning variable selection, and importance weighting.

\textbf{Limitations and Future Work.}
The population identification result relies on marginal detectability,
common support, and sufficient local conditioning information under
the stated causal and intervention assumptions. Practical performance
also depends on the quality of variable selection and weight estimation,
and RRV may miss conditional changes not reflected in the chosen
prediction loss. Future work could develop richer conditional
comparisons, improve alignment under limited support overlap,
and establish finite-sample guarantees for the complete procedure.
Extending the framework to latent confounding and causally nested
intervention targets would further broaden its applicability.

\clearpage

\bibliography{tmlr}

\clearpage
\appendix


\section{Theoretical Analysis}
\label{app:theoretical_analysis}

\subsection{Proof of Proposition~\ref{prop:population_identification}}
\label{app:population_proof}

We first establish a result for comparing conditional distributions
across the two regimes.
Throughout this subsection, densities are taken with respect
to a common product measure
$\mu=\bigotimes_{i=1}^d\mu_i$, with Lebesgue or counting
measures as appropriate.
All density identities are valid almost everywhere.

\textbf{Conditional distribution comparison.}
For $e\in\{\mathrm{obs},\mathrm{int}\}$, write
$p_i^e(x_i\mid\mathbf x_{\mathrm{Pa}_G(i)})$ (in abbreviation $p_i^e$)
for the conditional density induced by the structural
assignment of $X_i$.
The Markovian SCM assumptions give
\begin{equation}
    p^e(\mathbf x)
    =
    \prod_{i=1}^d
    p_i^e(x_i\mid\mathbf x_{\mathrm{Pa}_G(i)}).
    \label{eq:app_scm_factorization}
\end{equation}
For every $i\notin\mathcal T$, the conditional density of $X_i$
given its parents is unchanged across regimes:
$p_i^{\mathrm{obs}}=p_i^{\mathrm{int}}$.
Let $B_j^G$ denote the set of parents, children, and spouses of
$X_j$:
\[
    B_j^G
    :=
    \mathrm{Pa}_G(j)\cup
    \mathrm{Ch}_G(j)\cup
    \mathrm{Sp}_G(j).
\]
Suppose that the conditional density of each child of $X_j$ given
its parents is unchanged across regimes. Under the stated positivity
assumption, for any $B_j^G\subseteq S\subseteq[d]\setminus\{j\}$,
the conditional distribution of $X_j$ given $\mathbf X_S$ is
unchanged if and only if its conditional density given its parents
is unchanged:
\begin{equation}
    P_{X_j\mid\mathbf X_S}^{\mathrm{obs}}
    =
    P_{X_j\mid\mathbf X_S}^{\mathrm{int}}
    \quad\Longleftrightarrow\quad
    p_j^{\mathrm{obs}}=p_j^{\mathrm{int}}.
    \label{eq:app_local_equivalence}
\end{equation}

Canceling the factors in
Equation~\ref{eq:app_scm_factorization} that do not involve
$x_j$ gives
\begin{equation}
    p^e(x_j\mid\mathbf x_{-j})
    =
    \frac{
        p_j^e(x_j\mid\mathbf x_{\mathrm{Pa}_G(j)})
        \Lambda_j^e(x_j;\mathbf x_{-j})
    }{
        Z_j^e(\mathbf x_{-j})
    },
    \label{eq:app_full_conditional}
\end{equation}
where
\begin{align*}
    \Lambda_j^e(u;\mathbf x_{-j})
    &:=
    \prod_{c\in\mathrm{Ch}_G(j)}
    p_c^e\!\left(
        x_c\mid\mathbf x_{\mathrm{Pa}_G(c)}^{(j\leftarrow u)}
    \right),\\
    Z_j^e(\mathbf x_{-j})
    &:=
    \int_{\mathcal X_j}
    p_j^e(u\mid\mathbf x_{\mathrm{Pa}_G(j)})
    \Lambda_j^e(u;\mathbf x_{-j})\,d\mu_j(u).
\end{align*}
Here,
$\mathbf x_{\mathrm{Pa}_G(c)}^{(j\leftarrow u)}$
denotes the parent vector with the coordinate of $X_j$
replaced by $u$; an empty product is defined as one.
Common-support positivity ensures that these expressions
are positive and finite for almost every conditioning value.

The right-hand side of
Equation~\ref{eq:app_full_conditional} depends on
$\mathbf x_{-j}$ only through $\mathbf x_{B_j^G}$.
Denote this conditional density by
$h_j^e(x_j\mid\mathbf x_{B_j^G})$.
By the tower property of conditional expectation, whenever
$S\supseteq B_j^G$,
\begin{equation}
    p^e(x_j\mid\mathbf x_S)
    =
    h_j^e(x_j\mid\mathbf x_{B_j^G}).
    \label{eq:app_blanket_superset}
\end{equation}
The conditioning marginals in both regimes have strictly
positive densities on the same product support.
Thus, equality of the conditional laws given $S$ is
equivalent to equality of the full conditional laws.

Since the conditional densities of the children of $X_j$ given
their parents are unchanged across regimes,
$\Lambda_j^{\mathrm{obs}}
=\Lambda_j^{\mathrm{int}}=:\Lambda_j$.
If the full conditional laws coincide, canceling the common
positive factor in Equation~\ref{eq:app_full_conditional}
yields
\[
    p_j^{\mathrm{int}}(x_j\mid\mathbf x_{\mathrm{Pa}_G(j)})
    =
    \frac{Z_j^{\mathrm{int}}(\mathbf x_{-j})}
         {Z_j^{\mathrm{obs}}(\mathbf x_{-j})}
    p_j^{\mathrm{obs}}(x_j\mid\mathbf x_{\mathrm{Pa}_G(j)}).
\]
For fixed $\mathbf x_{-j}$, the proportionality factor is
independent of $x_j$.
Integrating over the entire coordinate support
$\mathcal X_j$, both mechanism densities integrate to one,
so this factor equals one.
Fubini's theorem then gives equality of the two kernels
for almost every parent value.
The converse follows directly from
Equation~\ref{eq:app_full_conditional}, establishing
Equation~\ref{eq:app_local_equivalence}.
The Cartesian-product support is used here to normalize
each mechanism density while keeping the other
coordinates fixed.

\begin{proof}[Proof of Proposition~\ref{prop:population_identification}]
Let $\mathrm{De}_G(\mathcal T)$ denote the union of the
strict descendants of the intervention targets.
If $j\notin\mathcal T\cup\mathrm{De}_G(\mathcal T)$,
neither $X_j$ nor any of its ancestors is targeted.
The SCM on its ancestral set is therefore unchanged,
which implies
$P_{X_j}^{\mathrm{obs}}=P_{X_j}^{\mathrm{int}}$.
The marginal detectability assumption states that every
intervention target undergoes a marginal distribution shift:
\[
    P_{X_j}^{\mathrm{obs}}\neq P_{X_j}^{\mathrm{int}},
    \qquad \text{for every } j\in\mathcal T.
\]
Together, these observations imply
\begin{equation}
    \mathcal T\subseteq\mathcal C
    \subseteq
    \mathcal T\cup\mathrm{De}_G(\mathcal T).
    \label{eq:app_candidate_descendants}
\end{equation}

Next, we show that no candidate has an intervention target as a
child. Suppose that $j\in\mathcal C$ and
$c\in\mathrm{Ch}_G(j)\cap\mathcal T$. If $j\in\mathcal T$, the
edge $j\to c$ contradicts the assumption that no intervention
target is an ancestor of another.

If $j\notin\mathcal T$, Equation~\ref{eq:app_candidate_descendants}
implies that $j$ is a descendant of some target $t\in\mathcal T$.
Combining the directed path $t\to\cdots\to j$ with the edge
$j\to c$ creates a directed cycle if $t=c$. If $t\neq c$, it makes
$t$ an ancestor of another target, again contradicting the
assumption. Therefore,
\begin{equation}
    \mathrm{Ch}_G(j)\cap\mathcal T=\varnothing,
    \qquad j\in\mathcal C.
    \label{eq:app_unchanged_children}
\end{equation}

Normal-regime faithfulness gives $B_j=B_j^G$.
Although interventional faithfulness is not required,
the interventional distribution still factorizes over
$G$, so Equation~\ref{eq:app_blanket_superset} applies
in both regimes.

For $j\in\mathcal C\setminus\mathcal T$, neither $X_j$ nor any of
its children is an intervention target. Their conditional densities
given their respective parents therefore remain unchanged across
regimes.
Equation~\ref{eq:app_local_equivalence} therefore gives
\[
    P_{X_j\mid\mathbf X_{S_j}}^{\mathrm{obs}}
    =
    P_{X_j\mid\mathbf X_{S_j}}^{\mathrm{int}},
\]
so $\mathcal R_{\mathrm{cds}}\subseteq\mathcal T$.

For $j\in\mathcal T$, the assumption that no intervention target
is an ancestor of another implies that none of the ancestors of
$X_j$ is targeted. Therefore, the joint distribution of its parents
is unchanged: $P_{\mathbf X_{\mathrm{Pa}_G(j)}}^{\mathrm{obs}}=P_{\mathbf{X}_{\mathrm{Pa}_G(j)}}^{\mathrm{int}}.$

If $p_j^{\mathrm{obs}}=p_j^{\mathrm{int}}$, integrating this common
conditional density over the parent distribution would imply
$P_{X_j}^{\mathrm{obs}}=P_{X_j}^{\mathrm{int}}$, contradicting
marginal detectability. Hence,
$p_j^{\mathrm{obs}}\neq p_j^{\mathrm{int}}$.
Marginal detectability also gives $j\in\mathcal C$. By
Equation~\ref{eq:app_unchanged_children}, none of its children is
targeted, so their conditional densities given their respective
parents are unchanged. Equation~\ref{eq:app_local_equivalence}
therefore implies that the conditional distribution of $X_j$ given
$\mathbf X_{S_j}$ changes across regimes. As this holds for every
target, $\mathcal T\subseteq\mathcal R_{\mathrm{cds}}$.

Together with the reverse inclusion established above, this gives
$\mathcal R_{\mathrm{cds}}=\mathcal T$. Finally, every target is
an ancestor of the outcome by assumption, so
$\mathcal T=\mathcal R_Y$.
\end{proof}

\subsection{Motivation for SRDO-Based Local Conditioning}
\label{app:srdo_selection}

We analyze exact independence reweighting to explain how it guides
the selection of conditioning variables. The key properties are that it preserves the target's conditional distribution given all predictors while removing dependencies among predictors that can
give rise to spurious predictive associations.

Our goal is to construct local conditioning sets that retain
sufficient information for identification while limiting the
dimension of subsequent conditional comparisons.
However, feature importances from a fitted learner may favor
correlated proxies or distribute importance across redundant
predictors, making threshold-based selection less reliable.
Independence reweighting is intended to reduce these proxy
associations while preserving the target's conditional
relationship with the predictors.

For a fixed prediction target $X_j$, let $Q_j$ denote the
distribution obtained by reweighting $P^{\mathrm{obs}}$ with
the exact independence ratio $a_j(\mathbf X_{-j})$ in
Equation~\ref{eq:independence_weights}.
Under the positivity and product-support conditions,
\begin{align}
    p_{Q_j}(x_j,\mathbf x_{-j})
    &=
    a_j(\mathbf x_{-j})
    p^{\mathrm{obs}}(x_j,\mathbf x_{-j})\notag\\
    &=
    p^{\mathrm{obs}}(x_j\mid\mathbf x_{-j})
    \prod_{k\ne j}p_{X_k}^{\mathrm{obs}}(x_k).
    \label{eq:app_srdo_distribution}
\end{align}
This density integrates to one and makes the predictors
mutually independent.
Because the weights depend only on the predictors, the
target's conditional distribution remains unchanged:
\begin{equation}
    Q_j(X_j\mid\mathbf X_{-j})
    =
    P^{\mathrm{obs}}(X_j\mid\mathbf X_{-j}).
    \label{eq:app_srdo_conditional_preservation}
\end{equation}

Moreover, the Markov-boundary property gives
\[
p^{\mathrm{obs}}(x_j\mid x_{-j})
=
p^{\mathrm{obs}}(x_j\mid x_{B_j}).
\]
Equation 24 therefore factorizes as
\[
p_{Q_j}(x_j,x_{-j})
=
\left[
p^{\mathrm{obs}}(x_j\mid x_{B_j})
\prod_{k\in B_j}p^{\mathrm{obs}}_{X_k}(x_k)
\right]
\prod_{k\notin B_j\cup\{j\}}
p^{\mathrm{obs}}_{X_k}(x_k).
\]
Thus, under exact independence reweighting and the stated support conditions, predictors outside \(B_j\) are jointly independent of \((X_j,X_{B_j})\). Their predictive associations arising through dependence on boundary variables are therefore removed, while the target’s full conditional distribution is preserved. This provides a population-level motivation for applying SRDO before importance-based feature selection.

These properties provide a rationale for using independence reweighting to make feature selection less sensitive to correlated proxies. In the implemented procedure, the selected conditioning sets additionally depend on the learner, feature-importance measure, and selection threshold. We assess the usefulness of this practical combination through the conditioning-set ablations in Section~\ref{subsec:ablation_studies}.

\section{StableRCA Implementation Details}
\label{appendix:stablerca_implementation}

We describe the implementation of the three steps in
Subsection~\ref{subsec:stablerca_practical_implementation}.
The inputs are the normal dataset $\mathcal{D}_{\mathrm{obs}}$
and the anomalous dataset $\mathcal{D}_{\mathrm{int}}$,
represented as data frames with the same $d$ variable
columns and $N_{\mathrm{obs}}$ and $N_{\mathrm{int}}$
samples, respectively.
The designated outcome $Y=X_o$ is excluded from the
root-cause candidates.
The algorithm returns a dictionary mapping candidate
variable names to their RRV scores, which determine
the root-cause ranking.

For local modeling, we split the normal data into
$\mathcal{D}_{\mathrm{obs}}^{\mathrm{tr}}$ and
$\mathcal{D}_{\mathrm{obs}}^{\mathrm{te}}$.
Both local conditioning set selection and local predictor
fitting use the training set, while the evaluation set
is used to estimate the normal prediction risk.

\textbf{Marginal distribution shift filtering.}
For each $j\in[d]\setminus\{o\}$, we compare the marginal
distributions of $X_j$ using $\mathcal{D}_{\mathrm{obs}}$
and $\mathcal{D}_{\mathrm{int}}$.
We determine the variable type from the normal data:
a variable is treated as discrete if it is categorical
or has fewer than 10 distinct observed values;
otherwise, it is treated as continuous.
We apply the two-sample Kolmogorov--Smirnov test to
continuous variables and the $\chi^2$ test to discrete
variables.
The per-variable significance level is
$\alpha=0.001$ throughout our experiments.
Variables with $p$-values satisfying $\pi_j\leq\alpha$
are retained in the candidate set $\widehat{\mathcal C}$
defined in Equation~\ref{eq:empirical_screening}.

\textbf{Local conditioning set selection.}
For each candidate $j\in\widehat{\mathcal C}$, we treat
$X_j$ as the prediction target and all remaining variables
$\mathbf X_{-j}$ as the initial predictors.
Using $\mathcal{D}_{\mathrm{obs}}^{\mathrm{tr}}$,
SRDO~\citep{shen2020stable} estimates the independence
weights $\widehat a_j$ in
Equation~\ref{eq:independence_weights} by contrasting
the observed predictor distribution with samples formed
by independently resampling each predictor.
These weights are intended to reduce reliance on
correlated proxies during feature selection.

We fit a CatBoost regressor for continuous targets
or a CatBoost classifier for discrete
targets~\citep{10.5555/3327757.3327770},
using $\widehat a_j$ as sample weights.
Let $\widehat I_{jk}$ denote the fitted model's built-in
importance score for predictor $X_k$.
We select
\[
    \widehat S_j
    :=
    \left\{
        k\in[d]\setminus\{j\}:
        \widehat I_{jk}
        >
        q\max_{\ell\ne j}\widehat I_{j\ell}
    \right\},
    \qquad q=0.15.
\]
We use a fixed relative-importance threshold \(q=0.15\), retaining predictors whose importance exceeds \(15\%\) of the largest importance for the corresponding target. This rule allows the number of selected predictors to vary across variables according to their importance profiles. The relatively permissive cutoff is intended to retain informative predictors with weaker contributions while filtering those assigned very low importance. It provides a practical trade-off between inclusive conditioning and the dimensionality of subsequent prediction and covariate alignment. We apply the same threshold throughout the experiments and evaluate the resulting conditioning sets through ablation studies.

\textbf{Conditional distribution shift scoring.}
For each candidate $j\in\widehat{\mathcal C}$, we fit
a local predictor $\widehat f_j$ on
$\mathcal{D}_{\mathrm{obs}}^{\mathrm{tr}}$ using only
the selected conditioning variables
$\mathbf X_{\widehat S_j}$.
We again use CatBoost regression for continuous targets
and CatBoost classification for discrete targets.
The fitted predictor is kept fixed when evaluated on
$\mathcal{D}_{\mathrm{obs}}^{\mathrm{te}}$ and
$\mathcal{D}_{\mathrm{int}}$.
We compute mean squared error for continuous targets
and cross-entropy loss from predicted class probabilities
for discrete targets.

To account for changes in the distribution of the
conditioning variables, we estimate the alignment
ratio $w_j$ in Equation~\ref{eq:alignment_weights}
using an auxiliary binary classifier that distinguishes
normal from anomalous samples based only on
$\mathbf X_{\widehat S_j}$.
The estimated regime probabilities are converted to
density ratios, accounting for the class proportions
used to train the classifier,
following standard classifier-based density-ratio
estimation~\citep{10.5555/1314498.1390324}.
These alignment weights $\widehat w_j$ are applied
to the anomalous prediction losses.
They serve a different purpose from the independence
weights $\widehat a_j$ used during conditioning set
selection.

We compute the empirical normal and aligned anomalous
risks as in Equation~\ref{eq:empirical_aligned_risks}.
For candidates with $\widehat R_j^{\mathrm{obs}}>0$,
we obtain the RRV score using
Equation~\ref{eq:rca_score}.
Candidates are ranked in descending order of their
scores, with larger relative increases in aligned
prediction risk providing stronger root-cause evidence.
The scores serve as a predictive proxy for conditional
distribution shifts.

\section{Dataset Description}\label{appendix:dataset_description}

\subsection{Synthetic Data Generation}\label{appendix:synthethic_data}
We generate synthetic tabular datasets using a causal graph-based data-generating mechanism, and create anomalies by intervening on a small set of causal variables (“root causes”). In each trial, we first sample a causal graph that specifies the high-level data-generating structure among variables, and instantiate a structural causal model by assigning node-wise structural equations and exogenous noise distributions. We then draw a set of observational (normal) samples for training, and generate abnormal samples under interventions on the selected root-cause variables. Along with the resulting observational and interventional datasets, we also record the ground-truth root-cause set, as well as a designated target node for evaluation.

\paragraph{Causal Graph Construction}We sample an undirected Erd\H{o}s--R\'enyi (ER) graph
with $d$ nodes and $2d$ edges~\citep{erdos59a}.
We then sample a uniformly random ordering of the nodes
and orient every edge from the earlier node to the later
node in this ordering.
The resulting directed acyclic graph (DAG), denoted by
$G$, specifies the causal structure of the SCM.

\paragraph{Structural Assignments and Normal Data Generation}
Variables are generated in topological order.
Graph-root variables are sampled independently from
$\mathcal{U}(0,1)$.
Each non-root variable follows an additive-noise model:
\begin{equation}
    X_j
    =
    \sum_{i\in\mathrm{Pa}_G(j)}
    f_{ij}(X_i)
    +
    \epsilon_j,
\end{equation}
where $f_{ij}$ is an edge-specific causal function and
$\epsilon_j$ is an exogenous noise variable.

Each edge function is independently chosen to be linear
or MLP-based with equal probability.
A linear function takes the form $f_{ij}(x)=w_{ij}x$,
where $|w_{ij}|\sim\mathcal{U}(0.25,1)$ and its sign
is independently chosen to be positive or negative
with equal probability.
Each MLP has two hidden layers with $10$ units per layer
and a scalar output.
All weights are independently initialized from
$\mathcal{N}(0,1)$, and no bias terms are used.
For each MLP, one activation function is sampled uniformly
from identity, hard sigmoid, ReLU, tanh, and softplus,
and used in both hidden layers.
All sampled functions remain fixed throughout normal
data generation.

\paragraph{Exogenous noise.}
For each SCM, one noise family is sampled uniformly
from Gaussian, Gumbel, uniform, and exponential
distributions and shared by all non-root variables.
Noise realizations are sampled independently across
variables and observations.
With the noise-control parameter $s=0.5$, the implemented
distributions are
\begin{equation}
    \epsilon_j\sim
    \begin{cases}
        \mathcal{N}(0,s^2),
        & \text{Gaussian},\\
        \mathrm{Gumbel}\!\left(
            \mathrm{loc}=0,\,
            \mathrm{scale}=\sqrt{6s}/\pi
        \right),
        & \text{Gumbel},\\
        \mathcal{U}(-\sqrt{3}s,\sqrt{3}s),
        & \text{uniform},\\
        \mathrm{Exp}(\mathrm{rate}=1/s)-s,
        & \text{exponential}.
    \end{cases}
\end{equation}
These expressions specify the parameterization of $s$
for each family.
The selected noise family and its parameters remain
unchanged under intervention.

\paragraph{Intervention targets and designated outcome.}
We consider scenarios with either one single intervention
target or multiple simultaneous intervention targets.
In the single-target setting, the target is sampled
uniformly from non-root nodes with more than one descendants. In the multiple intervention target setting, the selected non-root
variables satisfy the restriction that no target is
an ancestor of another.

For an intervention set $\mathcal{T}$, the designated
outcome must be a strict descendant of every target.
We therefore restrict attention to target sets with
a nonempty common-descendant set
\begin{equation}
    \mathcal{O}(\mathcal{T})
    :=
    \bigcap_{t\in\mathcal{T}}\mathrm{De}_G(t),
\end{equation}
and sample the outcome index $o$ uniformly from
$\mathcal{O}(\mathcal{T})$, setting $Y=X_o$.
Consequently,
$\mathcal{T}\subseteq\mathrm{An}_G(Y)$,
as required by the problem formulation in
Section~\ref{sec:problem_formulation}.

\paragraph{Normal and anomalous data generation.}
For each SCM, we generate $2{,}000$ normal observations
from the original structural assignments.
Anomalies are then introduced through soft-function
interventions.
Specifically, for every target $t\in\mathcal{T}$,
we independently resample all incoming edge functions
$f_{it}$ using the same function-selection and
parameter-initialization rules as above.
The replacement functions are held fixed while generating
$200$ anomalous observations.
The graph structure, noise distributions, and structural
functions of non-target variables remain unchanged.
Each anomalous observation therefore comes from the same
intervention regime, with independently sampled exogenous
noise realizations.

\paragraph{Evaluation protocol.}
Each main configuration is evaluated over three runs,
each comprising $20$ independently generated SCMs.
A new graph, structural functions, and datasets are
generated for each SCM trial.
Within a trial, all methods and graph-quality conditions
use the same normal observations, anomalous observations,
root-cause labels, and designated outcome.
Performance is first averaged over the $20$ trials
within each run.
We then report the mean and standard deviation of
the three run-level averages.
For five-target experiments, top-$k$ recovery metrics
are evaluated at $k=|\mathcal{T}|=5$.

\subsection{Real-World Datasets}
We provide detailed descriptions of the datasets used in our experimental evaluation. These datasets are derived from multiple real-world domains, including retail services, microservice architectures, manufacturing systems, and physical simulation environments, thereby reflecting a broad spectrum of practical deployment scenarios.

In our experiments, we utilize both normal and abnormal samples. For sample-level RCA baseline methods, including Score Ordering, Smooth Traversal, Traversal,  and Cholesky, RCA scores are computed by aggregating results across all abnormal samples and evaluating on their average. For graph-based methods, we use the available causal graph whenever possible. Specifically, the ground-truth causal graphs are used for \textit{ProRCA}, \textit{CausalMan}, and \textit{CausalChambers}. For \textit{RCAEval}, where ground-truth causal graphs are unavailable, we estimate a graph from normal data using XGES.

\textbf{ProRCA.} ProRCA is a semi-synthetic data provided by the ProRCA package, constructed via simulation to approximate real-world retail transaction processes in a complex, nonlinear manner \citep{dawoud2025prorcacausalpythonpackage}. The dataset is designed to model realistic business systems into which different types of anomalies can be systematically injected. It comprises 13 variables, with "PROFIT" designated as the outcome variable. The ground-truth causal graph underlying the data-generating process is illustrated in Figure~\ref{fig:prorca}. A total of five anomaly types can be injected to generate abnormal samples: \textit{ExcessiveDiscount}, \textit{COGs}, \textit{FulfillmentSpike}, \textit{ReturnSurge}, and \textit{ShippingDisruption}. We apply four of these anomaly types with predefined anomaly strengths. Additional dataset statistics and configuration details are reported in Table~\ref{tab:prorca_info}. 

\begin{figure}
    \centering
    \includegraphics[width=\linewidth]{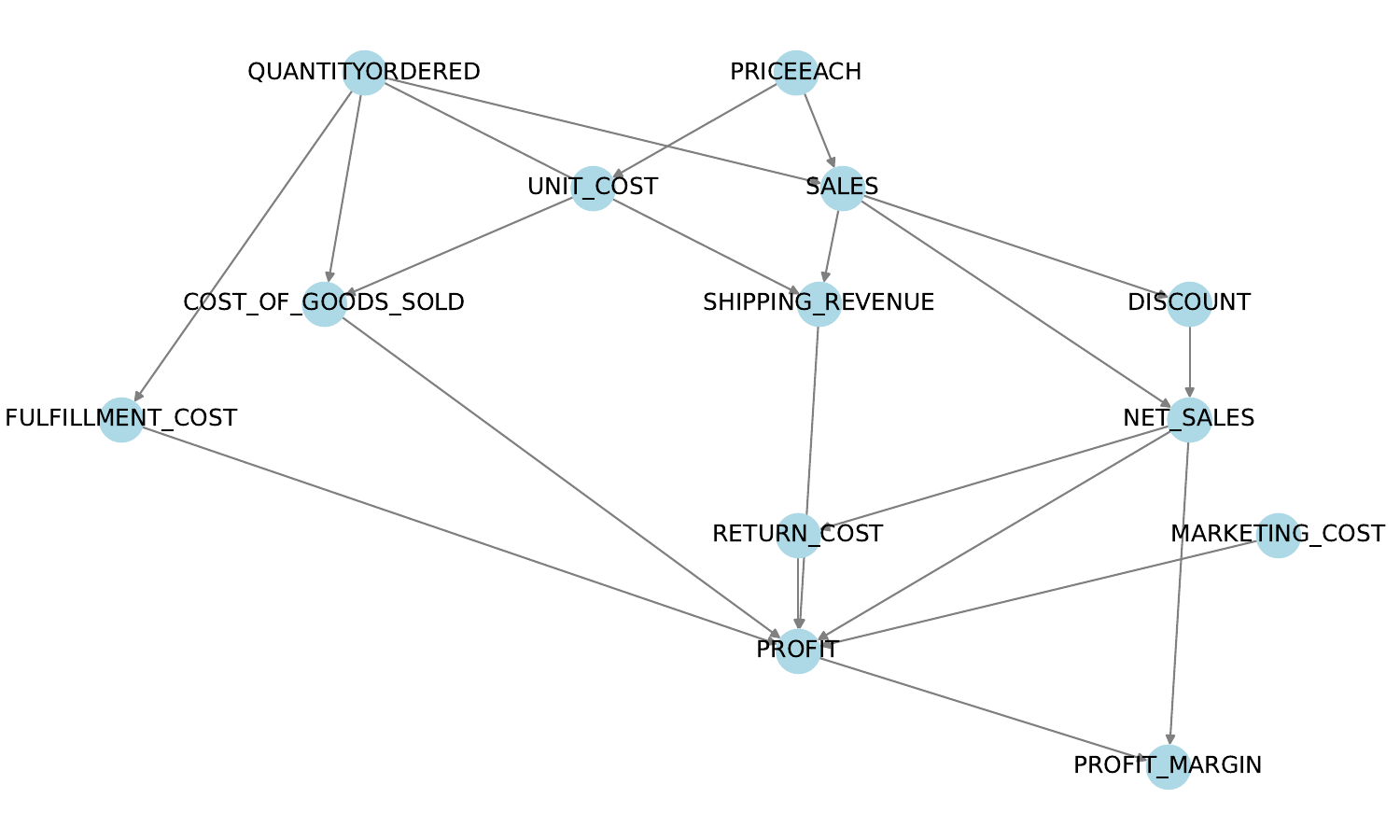}
    \caption{Causal graph of ProRCA dataset.}
    \label{fig:prorca}
\end{figure}

\begin{table}[ht!]
    \centering
    \caption{ProRCA data information.}
    \label{tab:prorca_info}
    \begin{tabular}{lccc}
        \toprule
        \textbf{Anomaly Type} & \textbf{Anomaly Strength} & \textbf{Normal Data Size} & \textbf{Abnormal Data Size} \\
        \midrule
        ExcessiveDiscount   & 0.2 & 300 & 60 \\
        FulfillmentSpike    & 3   & 200 & 40 \\
        ReturnSurge         & 10  & 300 & 60 \\
        ShippingDisruption  & 5   & 100 & 20 \\
        \bottomrule
    \end{tabular}
\end{table}



\textbf{CausalMan.} CausalMan is a manufacturing data simulator modeled after a real-world production line that assembles magnetic valves and hydraulic blocks in Hydraulic Units (HU) \citep{tagliapietra2025causalmanphysicsbasedsimulatorlargescale}. The simulator captures a wide range of linear and non-linear causal mechanisms, as well as challenging real-world behaviors such as discrete mode switches. The resulting CausalMan dataset is provided in two scales: \textbf{CausalMan Small}, which contains 157 variables, and \textbf{CausalMan Medium}, which contains 605 variables. In both settings, the data are generated under a causally sufficient assumption, where all relevant variables are observed. CausalMan represents a highly challenging real-world benchmark due to its complex non-linear dependencies and heterogeneous variable types, including continuous, categorical, and binary variables. The dataset includes 14 interventional settings targeting three variables, "PF\_M1\_T1\_Force", "PF\_M1\_T1\_Fmax" and "PF\_M1\_T1\_sgrad", that induce abnormal system behavior. Specifically, CausalMan Small contains 5 hard and 3 soft interventions, while CausalMan Medium includes 3 hard and 3 soft interventions. The outcome variable, "Sec\_C2\_Machine1\_ProcessResult" is binary and indicates whether the manufacturing process result is acceptable. Detailed information about all interventional configurations is provided in Table~\ref{tab:causalman_interventions}. For each intervention, both the normal (observational) dataset and the corresponding interventional (abnormal) dataset contain 47,400 samples. Due to computational constraints, we subsample 2,000 normal samples and 200 abnormal samples for each root cause analysis experiment.

\begin{table}[ht!]
    \centering
    \caption{CausalMan data information.}
    \label{tab:causalman_interventions}
    \resizebox{\linewidth}{!}{
    \begin{tabular}{llll}
        \toprule
        \textbf{Data Scale} & \textbf{Intervention Target Node} & \textbf{Intervention Type} & \textbf{Data Size} \\
        \midrule
        CausalMan Small  & PF\_M1\_T1\_Force & Hard intervention to value 16000 & 47,400 \\
        CausalMan Small  & PF\_M1\_T1\_Force & Hard intervention to value 17000 & 47,400 \\
        CausalMan Small  & PF\_M1\_T1\_Force & Hard intervention to value 30000 & 47,400 \\
        CausalMan Small  & PF\_M1\_T1\_Force & Soft intervention to $\mathcal{N}(17000, 3000)$ & 47,400 \\
        CausalMan Small  & PF\_M1\_T1\_Fmax  & Hard intervention to value 18500 & 47,400 \\
        CausalMan Small  & PF\_M1\_T1\_Fmax  & Soft intervention to $\mathcal{N}(18500, 3000)$ & 47,400 \\
        CausalMan Small  & PF\_M1\_T1\_sgrad & Hard intervention to value 20 & 47,400 \\
        CausalMan Small  & PF\_M1\_T1\_sgrad & Soft intervention to $\mathcal{N}(20, 4)$ & 47,400 \\
        CausalMan Medium & PF\_M1\_T1\_Force & Hard intervention to value 17000 & 47,400 \\
        CausalMan Medium & PF\_M1\_T1\_Force & Soft intervention to $\mathcal{N}(17000, 3000)$ & 47,400 \\
        CausalMan Medium & PF\_M1\_T1\_Fmax  & Hard intervention to value 18500 & 47,400 \\
        CausalMan Medium & PF\_M1\_T1\_Fmax  & Soft intervention to $\mathcal{N}(18500, 3000)$ & 47,400 \\
        CausalMan Medium & PF\_M1\_T1\_sgrad & Hard intervention to value 20 & 47,400 \\
        CausalMan Medium & PF\_M1\_T1\_sgrad & Soft intervention to $\mathcal{N}(20, 4)$ & 47,400 \\
        \bottomrule
    \end{tabular}
    }
\end{table}

\begin{figure}
    \centering
    \includegraphics[width=\linewidth]{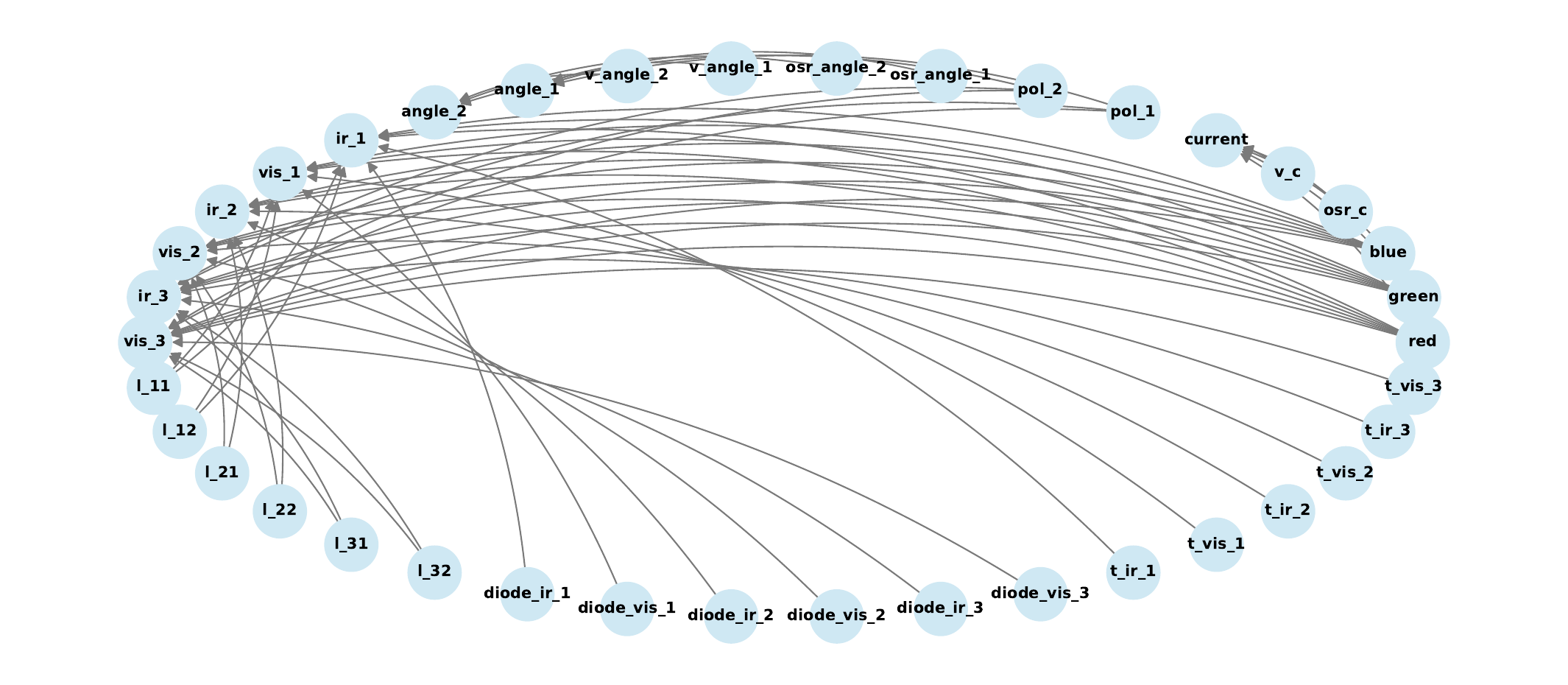}
    \caption{Causal graph of Causal-Chamber dataset.}
    \label{fig:causalchamber}
\end{figure}

\textbf{Causal-Chamber.} The Causal Chambers are real-world physical system simulators designed to generate data governed by known physical laws with fully specified ground-truth data-generating processes \citep{gamella2025chamber}. Two physical systems are provided: a wind tunnel and a light tunnel. The wind tunnel consists of a chamber equipped with two controllable fans that drive airflow, along with barometers that measure air pressure at multiple locations. The light tunnel comprises a controllable light source at one end and two linear polarizers mounted on rotating frames. In this work, we use interventional data generated from the light tunnel system for root cause analysis experiments. The light tunnel system contains 38 variables and supports 58 distinct interventions with varying intervention targets and anomaly strengths. The ground-truth causal graph of the system is shown in Figure~\ref{fig:causalchamber}. We conduct experiments on 16 interventions applied to continuous variables. For each intervention, the outcome variable is randomly selected from the leaf nodes among the descendants of the intervention target. The normal data is generated as follows. The actuators are sampled as follows: 
$R, G, B \overset{\text{i.i.d.}}{\sim} \text{Unif}(\{0, \dots, 85\})$; $L_{11}, L_{12}, L_{21}, L_{22}, L_{31}, L_{32} \overset{\text{i.i.d.}}{\sim} \text{Unif}(\{0, \dots, 170\})$; and $\theta_1, \theta_2 \overset{\text{i.i.d.}}{\sim} \text{Unif}(\{-15, -14.9, \dots, 25\})$, where $\text{Unif}(S)$ denotes sampling uniformly from $S$. The sensor parameters are set as follows: $O_1, O_2, O_C = 1$; $R_1, R_2, R_C = 5$; $D_1^I, D_2^I, D_3^I = 2$; $D_1^V, D_2^V, D_3^V = 1$; and $T_1^V, T_2^V, T_3^V, T_1^I, T_2^I, T_3^I = 3$. This procedure yields 10,000 normal samples. Detailed information about the interventions is summarized in Table~\ref{tab:causalchamber}. Each interventional (abnormal) dataset contains 1,000 samples. Due to computational constraints, we subsample 2,000 normal samples and 200 abnormal samples for each root cause analysis experiment.

\begin{table}[ht!]
    \centering
    \caption{Causal-Chamber data information.}
    \label{tab:causalchamber}
    \resizebox{\linewidth}{!}{
    \begin{tabular}{llp{0.62\linewidth}}
        \toprule
        \textbf{Intervention Type} & \textbf{Data Size} & \textbf{Data Description} \\
        \midrule
        Uniform colors mid    & 1000 & Sample $R, G, B \overset{\text{i.i.d.}}{\sim} \mathrm{Unif}(\{86, \dots, 170\})$ \\
        Uniform colors strong & 1000 & Sample $R, G, B \overset{\text{i.i.d.}}{\sim} \mathrm{Unif}(\{171, \dots, 255\})$ \\
        Uniform pol mid       & 1000 & Sample $\theta_1, \theta_2 \overset{\text{i.i.d.}}{\sim} \mathrm{Unif}(\{26, 26.1, \dots, 65\})$ \\
        Uniform pol strong    & 1000 & Sample $\theta_1, \theta_2 \overset{\text{i.i.d.}}{\sim} \mathrm{Unif}(\{66, 66.1, \dots, 105\})$ \\
        Uniform led mid       & 1000 & Sample $L_{11}, L_{12}, L_{21}, L_{22}, L_{31}, L_{32} \overset{\text{i.i.d.}}{\sim} \mathrm{Unif}(\{171, \dots, 255\})$ \\
        \bottomrule
    \end{tabular}
    }
\end{table}

\begin{table}[ht!]
    \centering
    \caption{RCAEval data information. RCAEval contains benchmark failure cases from three microservice systems: Online Boutique (OB), Sock Shop (SS), and Train Ticket (TT).}
    \label{tab:rcaeval_info}
    \small
    \setlength{\tabcolsep}{4pt}
    \renewcommand{\arraystretch}{1.12}
    \resizebox{\linewidth}{!}{
    \begin{tabular}{llllcc}
        \toprule
        \textbf{Dataset} & \textbf{System} & \textbf{Fault Types} & \textbf{Telemetry Data} & \textbf{Cases} & \textbf{\# Metrics} \\
        \midrule
        RE1-OB & Online Boutique & CPU, MEM, DISK, DELAY, LOSS & Metrics only & 125 & 49--59 \\
        RE1-SS & Sock Shop       & CPU, MEM, DISK, DELAY, LOSS & Metrics only & 125 & 57--63 \\
        RE1-TT & Train Ticket    & CPU, MEM, DISK, DELAY, LOSS & Metrics only & 125 & 198--238 \\
        \midrule
        RE2-OB & Online Boutique & CPU, MEM, DISK, DELAY, LOSS, SOCKET & Metrics, logs, traces & 90 & 69--77 \\
        RE2-SS & Sock Shop       & CPU, MEM, DISK, DELAY, LOSS, SOCKET & Metrics, logs & 90 & 74--82 \\
        RE2-TT & Train Ticket    & CPU, MEM, DISK, DELAY, LOSS, SOCKET & Metrics, logs, traces & 90 & 340--376 \\
        \bottomrule
    \end{tabular}
    }
\end{table}

\textbf{RCAEval.} RCAEval is an open-source benchmark for RCA in microservice systems~\cite{pham2025rcaeval}. It contains 735 failure cases collected from three representative systems: Online Boutique, Sock Shop, and Train Ticket. The benchmark covers 11 fault types, including resource, network, and code-level failures, with annotated root-cause services and indicators for each case. It comprises three suites: RE1 provides metric-only data, while RE2 and RE3 include multi-source telemetry comprising metrics, logs, and traces, with RE3 focusing on code-level faults. In our experiments, we use RE1 and RE2, which contains 645 cases in total and report service-level Top-1 accuracy: a prediction is considered correct if the top-ranked variable belongs to the annotated root-cause service.

\section{Baseline Models}\label{appendix:baseline_models}

\begin{table}[t]
\centering
\caption{Comparison of representative root cause analysis (RCA) methods.}
\label{tab:rca_comparison}
\small
\setlength{\tabcolsep}{5pt}
\renewcommand{\arraystretch}{1.18}
\resizebox{\linewidth}{!}{
\begin{tabularx}{\linewidth}{
    l
    c
    c
    c
    >{\raggedright\arraybackslash}X
}
\toprule
\textbf{Method} 
& \textbf{Paradigm} 
& \textbf{Need Causal Graph} 
& \textbf{Analysis Level} 
& \textbf{Key Assumptions} \\
\midrule
$\varepsilon$-Diagnosis 
& Statistical 
& $\times$ 
& Population 
& Root causes exhibit distinguishable statistical properties between normal and abnormal periods \\

BARO
& Statistical
& $\times$
& Population
& Root-cause metrics exhibit detectable distributional changes between normal and abnormal periods \\

Traversal 
& Causal 
& \checkmark 
& Population 
& Anomalies propagate along causal paths from the outcome variable \\

CIRCA 
& Causal 
& \checkmark 
& Population 
& DAG, causal sufficiency, and faithfulness \\

RCD 
& Causal 
& $\times$ 
& Population 
& DAG, causal sufficiency, and faithfulness \\

Counterfactual 
& Causal 
& \checkmark 
& Sample 
& Known and invertible structural equations \\

Cholesky 
& Causal 
& $\times$ 
& Sample 
& Linear model, single intervention, and single root cause \\

Smooth Traversal 
& Causal 
& \checkmark 
& Sample 
& Single root cause and causal sufficiency \\

Score Ordering 
& Causal 
& $\times$ 
& Sample 
& Single root cause, unknown polytree structure, and causal sufficiency \\

RCG-0
& Causal
& Partial / learned
& Population
& Causal sufficiency, faithfulness, failure as intervention, and reliable marginal-CI-based partial graph estimation \\
\bottomrule
\end{tabularx}
}
\end{table}

We compare our method against a diverse set of representative RCA baselines, covering graph-based, statistics-based, and model-based approaches.

\textbf{$\epsilon$-Diagnosis\citep{10.1145/3308558.3313653}.} $\epsilon$-Diagnosis is a root cause localization method originally developed for microservice-based systems. It localizes root causes by performing energy-distance-based two-sample tests between normal and anomalous windows for each container-metric pair.
Metrics whose marginal distributions exhibit statistically significant deviations are reported as root-cause candidates.

\textbf{Traversal\citep{6848128, 10.1007/978-3-030-03596-9_1,liu2021microhecl}.} Traversal operates on a given causal graph and localizes root causes by selecting anomalous nodes that are upstream of the target, have no anomalous parents, and lie on paths composed solely of anomalous nodes.

\textbf{CI-RCA\citep{10.1145/3534678.3539041}.} CIRCA leverages a known causal graph to fit a linear structural model on normal data and detects root causes as nodes whose conditional behavior, given their parent variables, deviates most strongly during anomalies.

\textbf{RCD\citep{10.5555/3600270.3602529}.}
RCD formulates root cause analysis as a causal discovery problem by introducing a dummy fault node and identifying variables directly connected to it via conditional-independence-based structure learning on normal and anomalous data.

\textbf{Counterfactual\citep{budhathoki2022causal}.} Counterfactual uses a given SCM to quantify each node’s counterfactual contribution to the target anomaly via Shapley-based attribution analysis.

\textbf{Cholesky\citep{10.1093/jrsssb/qkaf066}.} Cholesky assumes a linear SEM and localizes root causes by analyzing the Cholesky decomposition of the covariance matrix under different variable permutations.

\textbf{Smooth Traversal\citep{orchard2025root}.} Smooth Traversal localizes the root cause by selecting the node with the largest positive difference between its anomaly score and that of its highest-scoring parent in the causal graph.

\textbf{Score Ordering\citep{orchard2025root}.} Score orderring identifies root causes by directly ranking variables according to their information-theoretic anomaly scores without using any structural information.

\textbf{RCG-0 \citep{ikram2025partialrca}}. RCG-0 is the $k=0$ variant of RCG that learns a partial causal structure using only marginal CI tests, then ranks root-cause candidates by CMI conditioned on possible parents, but its dense graph can force large conditioning sets and obscure the true root cause.

\textbf{BARO \citep{pham2024baro}.} BARO is an end-to-end microservice troubleshooting method that combines Multivariate Bayesian Online Change Point Detection for anomaly detection with a nonparametric statistical hypothesis-testing scorer to rank root-cause metrics and services.

\section{Oracle diagnostic of the RRV ranking rule}
\label{app:rrv_diagnostic}

We use a simple Gaussian chain to examine how RRV responds to different interventions. The conditioning sets and normal-regime predictors are known analytically, and covariate alignment is performed exactly. This allows us to study the behavior of RRV without errors from estimating these components. The diagnostic complements our experiments on the complete StableRCA pipeline.

\textbf{Experimental setup.}
Consider the normal-regime Gaussian chain
\begin{equation}
    X_1\sim\mathcal{N}(0,1),\qquad
    X_2=X_1+\varepsilon_2,\qquad
    Y=X_2+\varepsilon_Y,
\end{equation}
where $X_1$, $\varepsilon_2$, and $\varepsilon_Y$ are mutually
independent, and both noises are standard Gaussian.
We treat $Y$ as the monitored outcome and retain both
candidates $\{X_1,X_2\}$.
Each intervention replaces only the distribution of $X_1$ with
$\mathcal{N}(\mu,v)$, $v>0$, while preserving the downstream equations.
We consider a mean increase, a variance increase, a variance decrease,
and a combined mean and variance change that preserves the aligned
prediction risk, together with an unchanged control.
Every non-control setting has one intervention target that is an
ancestor of $Y$, changes all three marginal distributions, and
preserves positive joint densities on the common support
$\mathbb{R}^3$.

We use squared error, the oracle conditioning sets
$B_1=\{X_2\}$ and $B_2=\{X_1,Y\}$, and the fixed optimal
normal predictors
\begin{equation}
    f_1(X_2)=\frac{X_2}{2},\qquad
    f_2(X_1,Y)=\frac{X_1+Y}{2}.
\end{equation}
Both normal prediction risks equal $1/2$.
For each candidate, RRV is computed as
$s_j=(R_j^{\mathrm{int,align}}-R_j^{\mathrm{obs}})
/R_j^{\mathrm{obs}}$.
All conditioning sets and predictors are specified analytically.

\textbf{Population analysis.}
Under intervention, the target conditional distribution is
\begin{equation}
    X_1\mid X_2=z
    \sim\mathcal{N}\!\left(
        \frac{\mu+vz}{v+1},\frac{v}{v+1}
    \right).
\end{equation}
Evaluating the fixed normal predictor under this conditional and
averaging over the normal marginal $X_2\sim\mathcal{N}(0,2)$ gives
\begin{align}
    R_1^{\mathrm{int,align}}
    &=\frac12+\frac{v(v-1)+\mu^2}{(v+1)^2},
    \label{eq:appendix_rrv_risk}\\
    s_1&=\frac{2[v(v-1)+\mu^2]}{(v+1)^2}.
    \label{eq:appendix_rrv_population}
\end{align}
The propagated candidate retains the invariant conditional distribution
\begin{equation}
    X_2\mid X_1,Y
    \sim\mathcal{N}\!\left(\frac{X_1+Y}{2},\frac12\right),
\end{equation}
so $R_2^{\mathrm{int,align}}=1/2$ and $s_2=0$ in every setting.
Thus, the target score reflects the direction and magnitude of its
aligned prediction-risk change, while the invariant non-target has
zero population score.
For pure variance interventions ($\mu=0$), the sign of $s_1$ follows
the sign of $v-1$.

\textbf{Monte Carlo protocol.}
We verify the population calculations using 200 independent repetitions,
and 20,000 evaluation draws per risk.
In each repetition, one normal-chain sample supplies both baseline
risks.
Independently of this sample, we generate a separate aligned
evaluation sample for each candidate from
\begin{equation}
    q_j(\mathbf{s},x_j)
    =p_{\mathbf{X}_{B_j}}^{\mathrm{obs}}(\mathbf{s})
     p_{X_j\mid\mathbf{X}_{B_j}}^{\mathrm{int}}
     (x_j\mid\mathbf{s}).
\end{equation}
The conditioning variables therefore follow their normal distribution,
while the predicted variable follows its intervention conditional
distribution.
For $X_2$, directly sampling its conditional prediction residual from
$\mathcal{N}(0,1/2)$ yields the same aligned squared-loss distribution.
This protocol implements exact alignment by direct sampling, without
estimating importance weights or reweighting a shared anomalous sample.
We report means and standard deviations across repetitions and, for
each intervention setting, the frequency of the strict comparison
$\widehat s_1>\widehat s_2$.

\begin{table}[t]
    \centering
    \small
    \setlength{\tabcolsep}{4pt}
    \caption{Oracle analysis of signed RRV under Gaussian interventions.
    Population scores are computed analytically; Monte Carlo (MC)
    columns report mean $\pm$ standard deviation over 200 repetitions
    with 20,000 evaluation draws per risk.
    Both normal risks equal $0.5$, and population $s_2=0$ throughout.
    Only $X_1$ is intervened on in non-control rows.
    ``Target first'' denotes the frequency of the strict comparison
    $\widehat s_1>\widehat s_2$; the unchanged control has no target.}
    \label{tab:appendix_rrv_diagnostic}
    \begin{tabular}{lcrccr}
        \toprule
        Setting & $(\mu,v)$ & \shortstack{Population\\$s_1$}
        & MC $\widehat s_1$ & MC $\widehat s_2$
        & \shortstack{Target first\\(\%)} \\
        \midrule
        Unchanged control & $(0,1)$ & $0.00$
        & $-0.0003\pm0.0141$ & $0.0009\pm0.0138$ & --- \\
        Mean increase & $(1,1)$ & $0.50$
        & $0.5014\pm0.0198$ & $-0.0004\pm0.0142$ & $100.0$ \\
        Variance increase & $(0,4)$ & $0.96$
        & $0.9604\pm0.0257$ & $0.0008\pm0.0138$ & $100.0$ \\
        Variance decrease & $(0,1/4)$ & $-0.24$
        & $-0.2399\pm0.0102$ & $-0.0006\pm0.0136$ & $0.0$ \\
        Equal-risk change & $(\sqrt{3}/4,1/4)$ & $0.00$
        & $0.0010\pm0.0145$ & $0.0024\pm0.0133$ & $45.5$ \\
        \bottomrule
    \end{tabular}
\end{table}

\textbf{Results and interpretation.}
Table~\ref{tab:appendix_rrv_diagnostic} shows close agreement between
the Monte Carlo estimates and the population scores.
The unchanged control and the invariant non-target have estimated
scores close to zero.
Under the selected mean and variance increases, the target has
positive population RRV and ranks above $X_2$ in all 200 repetitions.
These settings illustrate how an aligned risk increase separates the
target from a propagated candidate whose conditional remains invariant.

Under the variance decrease, the target's aligned risk falls from
$0.5$ to $0.38$, giving $s_1=-0.24$.
The negative score reflects a reduction in prediction risk under the
fixed normal predictor.
Because signed RRV ranks candidates in descending order, the
invariant candidate with $s_2=0$ ranks above the target in every
repetition.

The combined mean and variance intervention illustrates how different
components of prediction risk can offset one another.
At $(\mu,v)=(\sqrt{3}/4,1/4)$, the target's conditional variance
decreases from $1/2$ to $1/5$, while the contribution from the
conditional-mean mismatch increases.
Their effects balance, giving
$R_1^{\mathrm{int,align}}=R_1^{\mathrm{obs}}=1/2$ and a population
tie at $s_1=s_2=0$ despite the conditional distribution change.
The observed $45.5\%$ target-first frequency reflects sampling
variation around this tie.
More generally, settings with $0<v<1$ and $\mu^2=v(1-v)$ have this
equal-risk property.

\textbf{Implications for RRV-based ranking.}
The diagnostic also clarifies how alternative score choices behave
in this example.
Taking absolute RRV would rank the target above $X_2$ under the
variance decrease, while preserving the tie under the equal-risk
intervention.
Evaluating the fixed normal Gaussian conditional density for $X_1$
with log loss yields the same aligned risk difference as squared error,
because its negative log density is
$\tfrac12\log\pi+(x_1-f_1(x_2))^2$.
The normalized RRV values would differ because their baseline risks
differ, but the signs and population ties would be unchanged.

These results clarify the connection between the conditional-shift
identification principle in Proposition~4.2 and its implementation
through a scalar prediction-risk score.
When all targets are retained, normal risks are positive, and
non-target conditionals are invariant, strict population separation
additionally requires a positive aligned risk increase for every
target.
RRV is therefore interpreted as a loss-dependent ranking score.
The oracle analysis makes this scoring condition explicit and
complements the empirical evaluation of the complete StableRCA
pipeline.

\section{More Experimental Results}
\subsection{Scalability to large scale systems}\label{appendix:scalability_large_system}

\begin{table}[t]
\centering
\rcatablesetup
\setlength{\tabcolsep}{3pt}
\caption{Top-5 precision/recall and runtime under increasing graph sizes. Precision and recall coincide with $k=5$. Runtime is reported in seconds.}
\label{tab:runtime_top5_scaling}
{\rcatablefont
\begin{tabular*}{\textwidth}{@{\extracolsep{\fill}}lcccccccc@{}}
\toprule
\textbf{Method}
& \multicolumn{2}{c}{$n=100$}
& \multicolumn{2}{c}{$n=200$}
& \multicolumn{2}{c}{$n=400$}
& \multicolumn{2}{c}{$n=800$} \\
\cmidrule(lr){2-3}
\cmidrule(lr){4-5}
\cmidrule(lr){6-7}
\cmidrule(lr){8-9}
& \textbf{P/R@5} & \textbf{Time}
& \textbf{P/R@5} & \textbf{Time}
& \textbf{P/R@5} & \textbf{Time}
& \textbf{P/R@5} & \textbf{Time} \\
\midrule
Score Ordering       & 0.93 & 537.60 & 0.95 & 1017.00 & 0.97 & 2216.00 & \textbf{1.00} & 4024.80 \\
Smooth Traversal     & 0.45 & 1.19   & 0.35 & 1.39    & 0.40 & 3.02    & 0.34 & 7.38 \\
Traversal            & 0.62 & 1.16   & 0.53 & 1.33    & 0.58 & 3.15    & 0.52 & 6.54 \\
Cholesky             & 0.90 & 230.00 & 0.83 & 756.20  & 0.75 & 2951.00 & 0.81 & 14994.00 \\
CIRCA                & 0.19 & 1.97   & 0.19 & 3.53    & 0.18 & 7.87    & 0.14 & 13.98 \\
Count. Attr.         & 0.24 & 239.40 & 0.21 & 480.70  & 0.25 & 1031.00 & 0.18 & 1928.80 \\
RCD                  & 0.44 & 16.17  & 0.58 & 24.80   & 0.75 & 21.97   & 0.45 & 29.86 \\
$\epsilon$-Diagnosis & 0.10 & 5.32   & 0.06 & 10.63   & 0.00 & 21.00   & 0.01 & 37.15 \\
BARO                 & 0.22 & 0.11   & 0.19 & 0.20    & 0.13 & 0.44    & 0.08 & 0.92 \\
RCG-0                & 0.30 & 13.11  & 0.25 & 53.49   & 0.34 & 237.07  & 0.22 & 1189.00 \\
\midrule
\textbf{\ourmethod}  & \textbf{0.97} & 144.70
                      & \textbf{0.99} & 240.50
                      & \textbf{1.00} & 570.20
                      & 0.99 & 860.90 \\
\bottomrule
\end{tabular*}
}
\end{table}

Table~\ref{tab:runtime_top5_scaling} reports Top-5 precision/recall and runtime on graphs with 100--800 nodes under linear SCMs, evaluated on a server with two Intel Xeon Gold 5320 CPUs at 2.20\,GHz. \ourmethod\ maintains near-perfect precision/recall of 0.97--1.00, achieving the highest accuracy for $n=100$--$400$ while requiring less runtime than both Score Ordering and Cholesky across all graph sizes. At $n=800$, \ourmethod\ achieves 0.99 precision/recall in 860.90 seconds, compared with 1.00 in 4024.80 seconds for Score Ordering and 0.81 in 14994.00 seconds for Cholesky. Traversal-based methods, CIRCA, $\epsilon$-Diagnosis, and BARO are faster but achieve substantially lower precision/recall, while RCD remains efficient but exhibits variable accuracy across graph sizes. Counterfactual Attribution and RCG-0 also provide lower accuracy and require more runtime than \ourmethod\ at $n=800$. These results demonstrate that \ourmethod\ offers a strong balance between root-cause identification accuracy and computational efficiency for large-scale RCA.

\subsection{Sensitivity With Respect to Abnormal Sample Size}\label{appendix:sensitivity_abnormal_sample}
\begin{table}[!t]
\centering
\rcatablesetup
\caption{Effect of abnormal sample size on StableRCA's Top-1 accuracy.}
\label{tab:stable_rca_abnormal_sample_size}
{\rcatablefont
\begin{tabular}{@{}rc@{}}
\toprule
\textbf{Abnormal samples} & \textbf{Top-1 Acc.} \\
\midrule
5    & 0.65 \\
50   & 0.80 \\
100  & 0.85 \\
200  & 0.80 \\
500  & 0.85 \\
1000 & 0.85 \\
2000 & 0.80 \\
\bottomrule
\end{tabular}
}
\end{table}

Table~\ref{tab:stable_rca_abnormal_sample_size} evaluates the sensitivity of \ourmethod\ to the number of abnormal samples. Top-1 accuracy increases from 0.65 to 0.85 as the abnormal sample size grows from 5 to 100, suggesting that additional observations support more reliable detection of marginal and conditional distribution shifts. With 100 or more abnormal samples, accuracy remains between 0.80 and 0.85, with no consistent improvement as the sample size increases further. These results suggest that \ourmethod\ performs well with a moderate number of abnormal samples and is relatively insensitive to further increases in abnormal sample size.

\subsection{Graph-based Methods Under Graph Corruption}\label{appendix:graph_corruption}
\begin{table}[t]
\centering
\rcatablesetup
\caption{Top-1 accuracy on 50-node ER graphs.
(a) Graph-based methods under different supplied-graph
conditions.
(b) Methods that do not require a supplied graph.}
\label{tab:precision_repeated_runs}
{\rcatablefont
\renewcommand{\arraystretch}{1.08}

\textbf{(a) Methods requiring a supplied graph}
\par\smallskip
\begin{tabular*}{\textwidth}{
    @{\extracolsep{\fill}}lccccc@{}
}
\toprule
\textbf{Method}
& \textbf{GT}
& \textbf{XGES}
& \textbf{Corr. 30\%}
& \textbf{Corr. 50\%}
& \textbf{Corr. 70\%} \\
\midrule
Traversal
& $0.68\pm0.03$ & $0.58\pm0.06$ & $0.48\pm0.08$
& $0.33\pm0.13$ & $0.30\pm0.13$ \\
CIRCA
& $0.63\pm0.03$ & $0.52\pm0.08$ & $0.45\pm0.05$
& $0.43\pm0.06$ & $0.35\pm0.05$ \\
Smooth Traversal
& $0.58\pm0.10$ & $0.55\pm0.05$ & $0.48\pm0.06$
& $0.32\pm0.03$ & $0.27\pm0.06$ \\
Counterfactual Attr.
& $0.30\pm0.13$ & $0.27\pm0.10$ & $0.18\pm0.08$
& $0.10\pm0.05$ & $0.17\pm0.12$ \\
\bottomrule
\end{tabular*}

\par\medskip
\textbf{(b) Methods without a supplied graph}
\par\smallskip
\begin{tabular}{@{}llc@{}}
\toprule
\textbf{Category}
& \textbf{Method}
& \textbf{Top-1 Acc.} \\
\midrule
\multirow{2}{*}{Statistical}
& $\epsilon$-Diagnosis & $0.07\pm0.12$ \\
& BARO                & $0.17\pm0.06$ \\
\midrule
\multirow{5}{*}{Causal}
& RCD                 & $0.83\pm0.03$ \\
& Cholesky            & $0.52\pm0.08$ \\
& Score Ordering      & $0.68\pm0.13$ \\
& RCG-0               & $0.25\pm0.09$ \\
& \textbf{\ourmethod} & $\mathbf{0.88\pm0.08}$ \\
\bottomrule
\end{tabular}
}
\end{table}

Table~\ref{tab:precision_repeated_runs} compares graph-based
methods under different supplied-graph conditions with
methods that do not require a supplied graph for RCA cases with a single intervention.
Overall, \ourmethod\ achieves the highest reported Top-1
accuracy of 0.88, demonstrating strong performance without
relying on an accurately specified causal graph.
Graph-based methods generally lose accuracy as graph
corruption increases.
For example, Traversal drops from 0.68 with the true graph
to 0.30 under 70\% corruption, while Smooth Traversal
decreases from 0.58 to 0.27, suggesting that incorrect
edges can mislead propagation-based RCA.
CIRCA also degrades, although it remains the strongest
graph-based baseline under 50\% and 70\% corruption.
Among methods without a supplied graph, RCD is the closest
competitor to \ourmethod, with an accuracy of 0.83.
Score Ordering and Cholesky achieve 0.68 and 0.52,
respectively, while $\epsilon$-Diagnosis, BARO, and RCG-0
perform substantially worse.
Finally, using the XGES-estimated graph reduces accuracy
by 0.03--0.11 compared with the true graph, but yields
higher accuracy than every corrupted-graph condition
for each graph-based method.
These results suggest that XGES retains useful information
for downstream RCA in this setting, while highlighting
the benefit of identifying root causes without depending
on supplied-graph quality.

\subsection{Detailed Experimental Results on Real-World Datasets}\label{appendix:more_exp_results}

\begin{figure}[ht!]
    \centering
    \includegraphics[width=\linewidth]{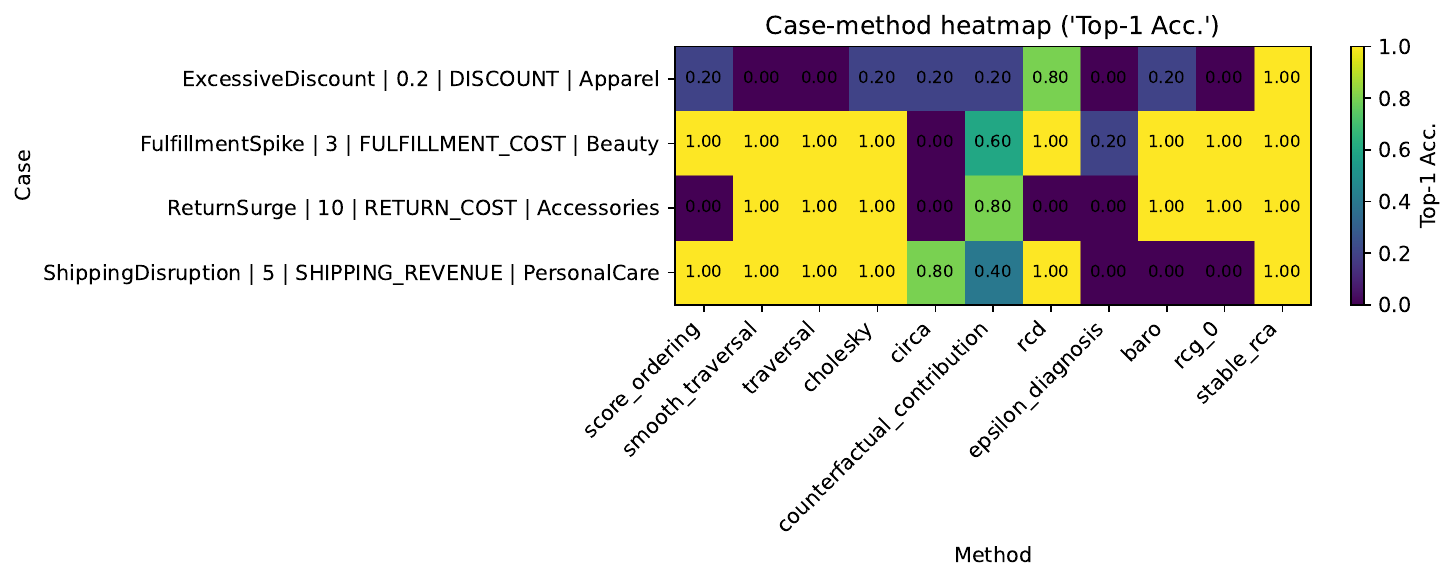}
    \caption{Case-level performance heatmap of different methods on ProRCA.}
    \label{fig:prorca_heatmap}
\end{figure}

\begin{figure}[ht!]
    \centering
    \includegraphics[width=\linewidth]{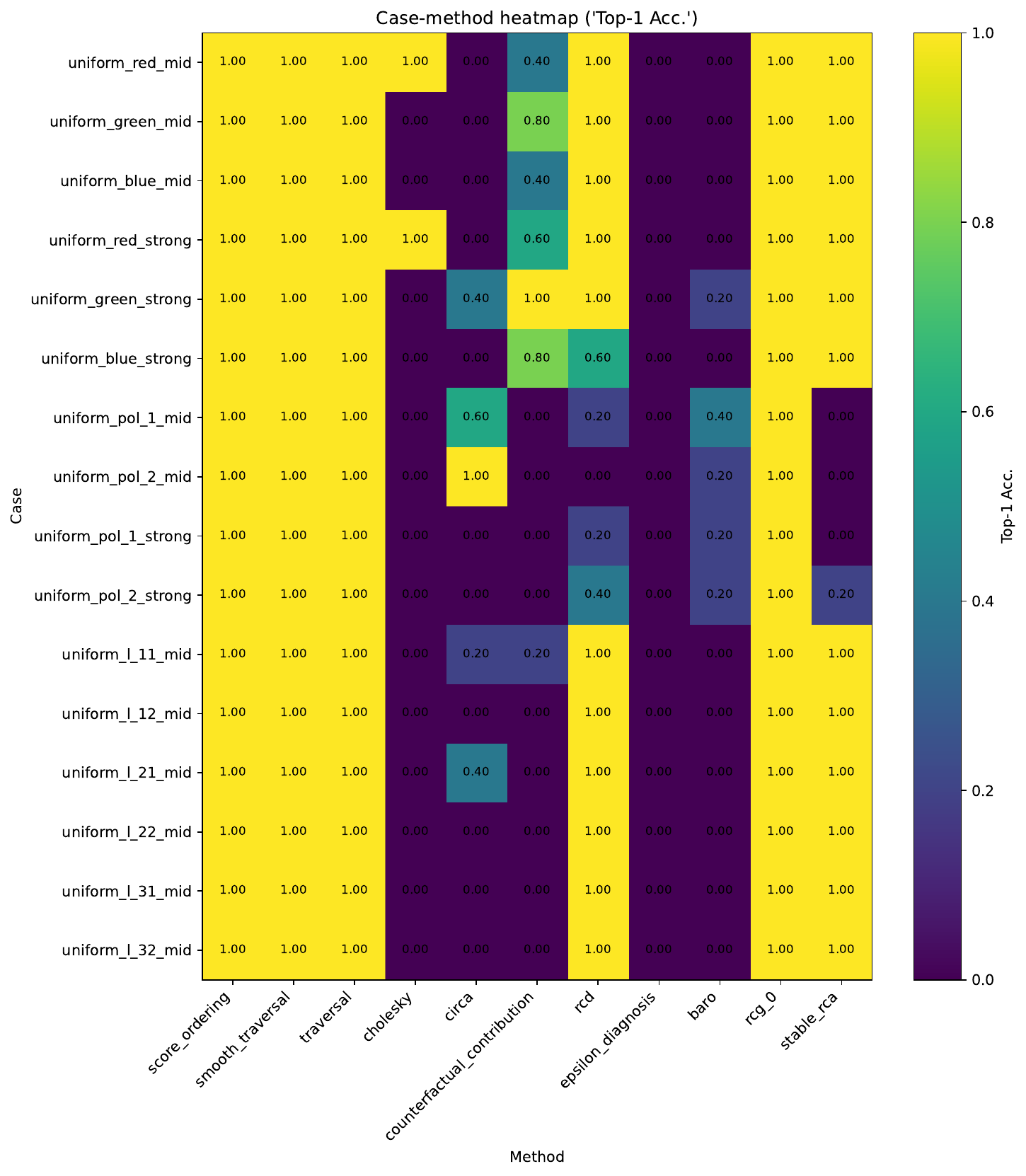}
    \caption{Case-level performance heatmap of different methods on Causal-Chamber.}
    \label{fig:causalchamber_heatmap}
\end{figure}

\begin{figure}[ht!]
    \centering
    \includegraphics[width=\linewidth]{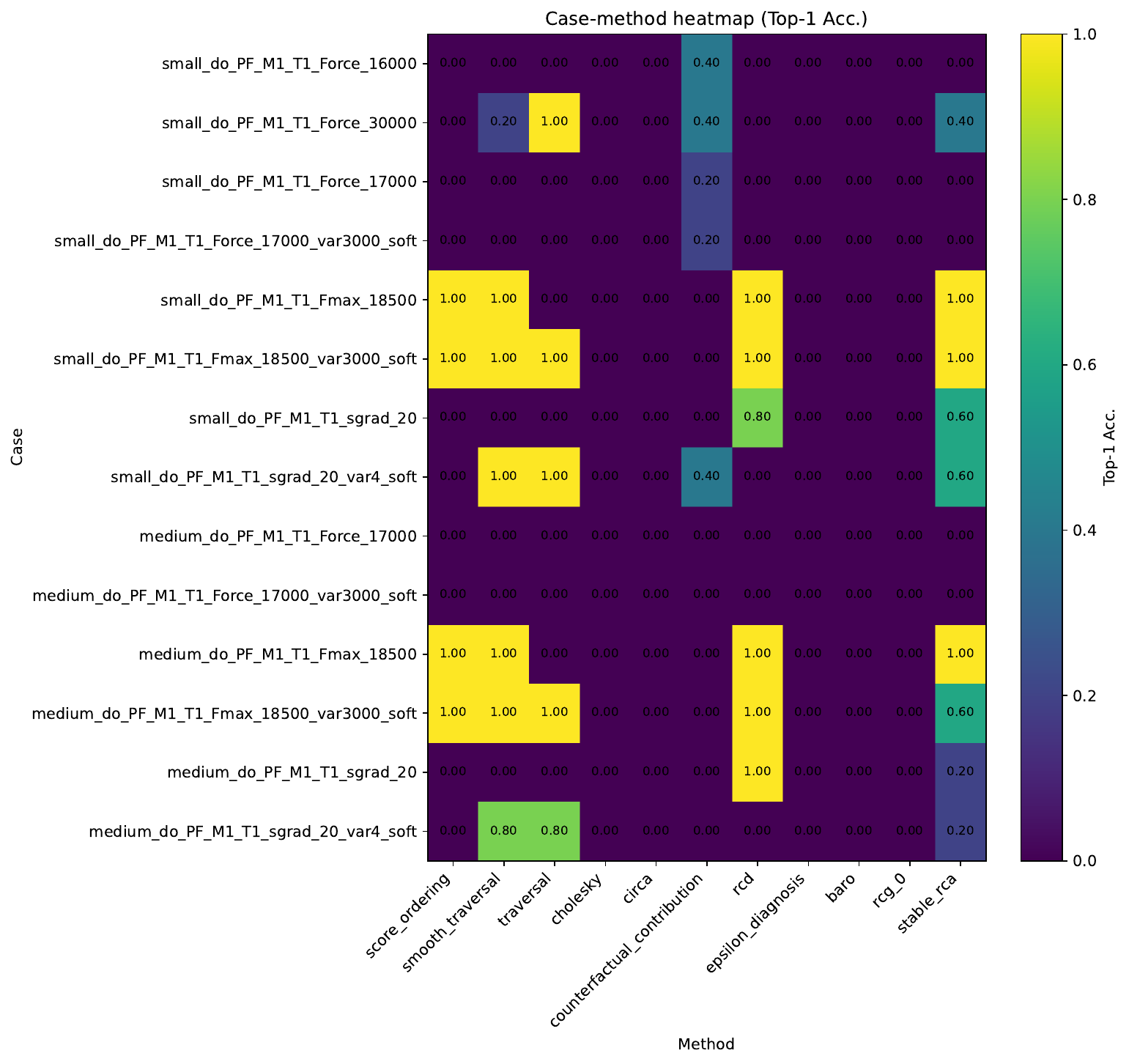}
    \caption{Performance heatmap of different methods on CausalMan.}
    \label{fig:causalman_heatmap}
\end{figure}

\begin{figure}[ht!]
    \centering
    \includegraphics[width=\linewidth]{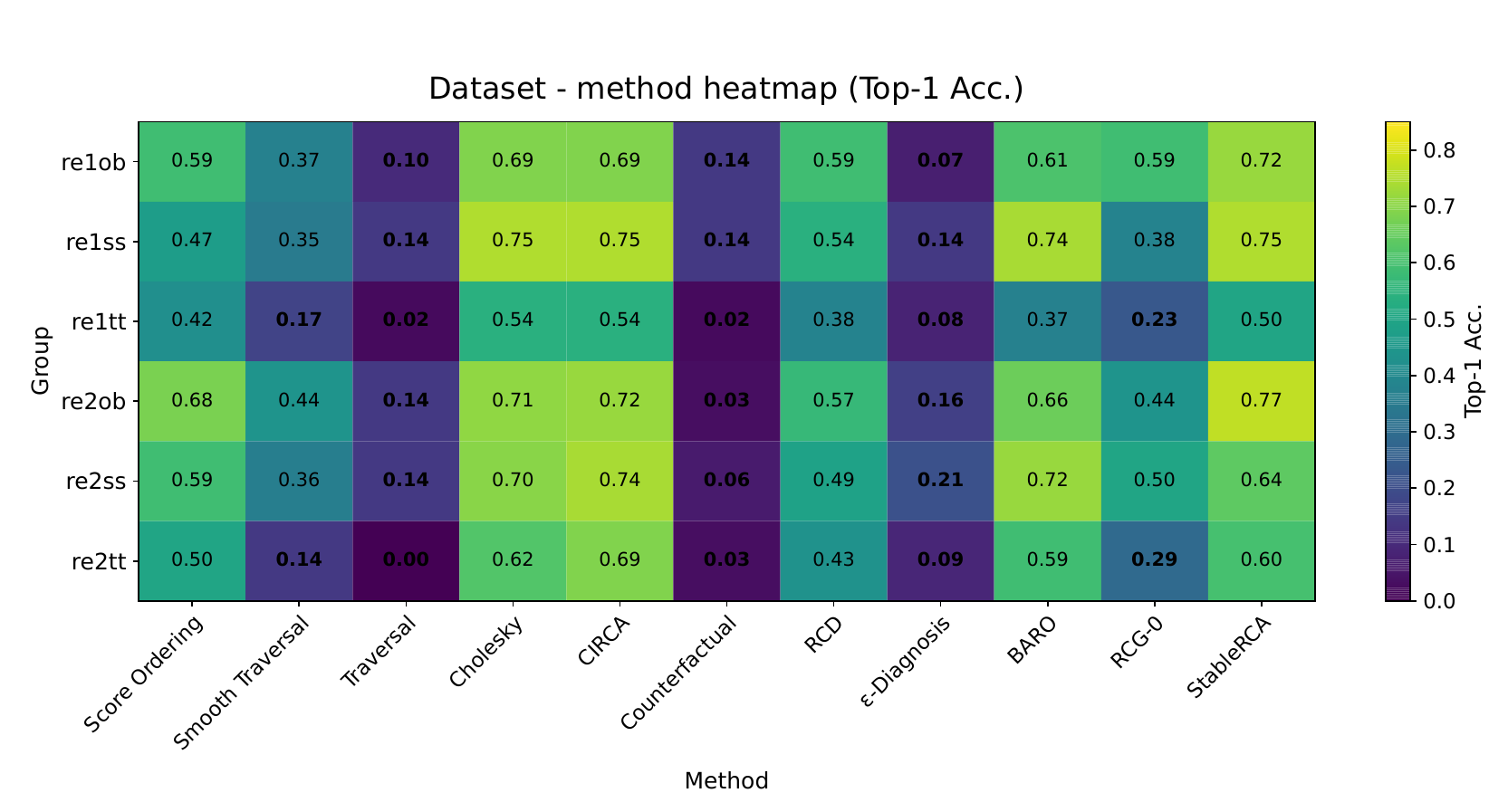}
    \caption{Performance heatmap of different methods on RCAEval, grouped by dataset.}
    \label{fig:rcaeval_dataset}
\end{figure}

\begin{figure}[ht!]
    \centering
    \includegraphics[width=\linewidth]{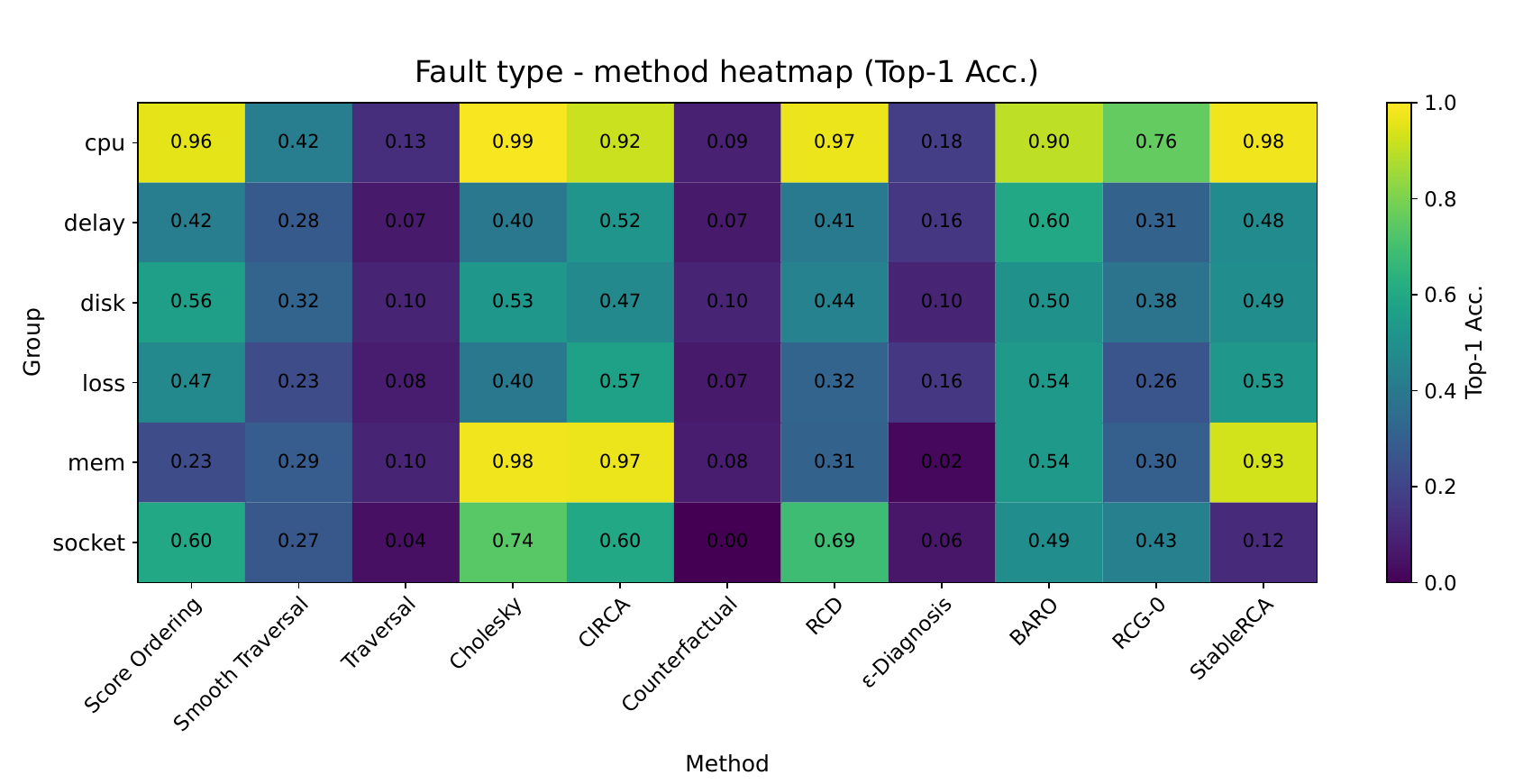}
    \caption{Performance heatmap of different methods on RCAEval, grouped by fault type.}
    \label{fig:rcaeval_fault}
\end{figure}

Here we present detailed experimental results on the four application benchmarks. Table~\ref{tab:real_world_precision_std} shows that \ourmethod\ achieves the highest average Top-1 accuracy. For example, Cholesky performs well on ProRCA and RCAEval but poorly on CausalChamber and CausalMan. Score Ordering and traversal-based methods achieve perfect accuracy on CausalChamber, while Smooth Traversal achieves the best result on CausalMan. On RCAEval, StableRCA remains competitive, achieving an accuracy of 0.66 compared with 0.69 for CIRCA and 0.67 for Cholesky, while outperforming the remaining baselines. These differences highlight the influence of dataset characteristics and intervention settings on the effectiveness of different RCA strategies. Figure~\ref{fig:prorca_heatmap}-Figure~\ref{fig:rcaeval_fault} report the detailed results on each case across datasets.




\end{document}